\documentclass[12pt]{amsart}
\usepackage{amssymb}
\usepackage{algorithm}
\usepackage{algpseudocode}
\usepackage{xspace}
\usepackage{needspace}
\usepackage{graphicx}
\usepackage{url}

\newcommand{\mymathop}[1]{\mathop{\texttt{#1}}}

\newcommand{\type}[1]{\ensuremath{\texttt{#1}}}

\newcommand{\mult}{\mathbin{\ast}}
\newcommand{\naturals}{\ensuremath{\mathbb{N}}\xspace}

\newcommand{\defined}{\underset{\text{def}}{\equiv}}
\newcommand{\dfn}[1]{{\bf #1}}

\newcommand{\mydot}{\raisebox{.05em}{$\,\bullet\,$}}
\newcommand{\cat}{\,.\,}
\newcommand{\size}[1]{\ensuremath{\left | {#1} \right |}}
\newcommand{\bigsize}[1]{\ensuremath{\bigl| {#1} \bigr|}}
\newcommand{\order}[1]{\ensuremath{{\mathcal O}(#1)}\xspace}
\newcommand{\Oc}{\order{1}}
\newcommand{\On}{\order{\var{n}}}
\newcommand{\inference}[2]{\genfrac{}{}{1pt}{}{#1}{#2}}

\newcommand{\isWellDef}{\mathop{\downarrow}}
\newcommand{\isIllDef}{\mathop{\uparrow}}

\newcommand{\unicorn}{\ensuremath{\bot}\xspace}

\newcommand{\tuple}[1]{\ensuremath{\left\langle #1 \right\rangle}}

\newcommand{\typed}[2]{\ensuremath{#1\mcolon#2}}
\newcommand{\Vtyped}[2]{\ensuremath{\var{#1}\mcolon#2}}
\newcommand{\VTtyped}[2]{\ensuremath{\var{#1}\mcolon\type{#2}}}

\newcommand{\lastix}[1]{\ensuremath{\##1}}
\newcommand{\Vlastix}[1]{\ensuremath{\lastix{\V{#1}}}}
\newcommand{\VlastElement}[1]{\Velement{#1}{\Vlastix{#1}}}

\newcommand{\Velement}[2]{\ensuremath{\mathop{\V{#1}}
    \mathopen{}\left[#2\right]\mathclose{}}}
\newcommand{\VVelement}[2]{\ensuremath{\mathop{\V{#1}}
    \mathopen{}\left[ \V{#2}
    \right]\mathclose{} }}

\newcommand{\seq}[1]{\ensuremath{\left[ #1 \right]}}
\newcommand{\seqB}[2]{\seq{#1 \mcoloncolon #2}}

\newcommand{\mcolon}{\mathrel:}
\newcommand{\mcoloncolon}{\mathrel{\vcenter{\hbox{$::$}}}}

\newcommand{\suchthat}{\mcoloncolon}
\newcommand{\quantify}[3]{
    \ensuremath{\mathrel{#1}#2 \; \suchthat \; #3}%
}

\newcommand{\rIota}{\mathop{\scalebox{1.5}{\rotatebox[origin=c]{180}{$\iota$}}}}
\newcommand{\iotaQOp}{\mathop{\rIota}}
\newcommand{\iotaQ}[2]{\quantify{\iotaQOp}{#1}{#2}}

\newcommand{\existQOp}{\mathop{\exists{}}}
\newcommand{\existQ}[2]{\quantify{\existQOp}{#1}{#2}}
\newcommand{\nexistQOp}{\mathop{\nexists{}}}
\newcommand{\nexistQ}[2]{\quantify{\nexistQOp}{#1}{#2}}
\newcommand{\univQOp}{\mathop{\forall{}\,}}
\newcommand{\univQ}[2]{\quantify{\univQOp}{#1}{#2}}

\newcommand{\fnMap}{\rightharpoonup}
\newcommand{\tfnMap}{\rightarrow}

\newcommand{\tfnMMap}{\ensuremath{
    \mathrel{\raise.75pt\hbox{\ensuremath{
          \raise1.19pt\hbox{\scalebox{.60}{$\ni$}} \mkern-12mu \rightarrow}}}
    }}

\newcommand{\subtract}{\,-\,}

\newcommand{\var}[1]{\ensuremath{\texttt{#1}}}
\newcommand{\V}[1]{\ensuremath{\texttt{\mbox{#1}}}}

\newcommand{\mname}[1]{\mbox{\sf #1}}
\newcommand{\Oname}[1]{\mname{#1}}
\newcommand{\OA}[2]{\ensuremath{%
    \mathop{\Oname{#1}}\mathopen{}\left(#2\right)\mathclose{}%
}}

\newcommand{\Hd}[1]{\OA{hd}{#1}}
\newcommand{\Tl}[1]{\Oname{tl}({#1})}
\newcommand{\HTl}[1]{\Hd{\Tl{#1}}}
\newcommand{\TTl}[1]{\Tl{\Tl{#1}}}
\newcommand{\TTTl}[1]{\Tl{\TTl{#1}}}

\newcommand{\de}{\mathrel{::=}}
\newcommand{\derives}{\Rightarrow}
\newcommand{\destar}
    {\mathrel{\mbox{$\:\stackrel{\!{\ast}}{\Rightarrow\!}\:$}}}
\newcommand{\deplus}
    {\mathrel{\mbox{$\:\stackrel{\!{+}}{\Rightarrow\!}\:$}}}
\newcommand{\derivg}[1]{\mathrel{\mbox{$\:\Rightarrow\:$}}}

\newcommand{\setOf}[1]{{{#1}^{\mbox{\normalsize $\ast$}}}}

\newcommand{\set}[1]{{\left\lbrace #1 \right\rbrace} }

\newcommand{\Vah}[1]{\ensuremath{\var{#1}_{\type{AH}}}}

\newcommand{\Vdr}[1]{\ensuremath{\var{#1}_{\type{DR}}}}

\newcommand{\Veim}[1]{\ensuremath{\var{#1}_{\type{EIM}}}}
\newcommand{\Veimt}[1]{\ensuremath{\var{#1}_{\type{EIMT}}}}

\newcommand{\Vpimt}[1]{\ensuremath{\var{#1}_{\type{PIMT}}}}
\newcommand{\es}[1]{\ensuremath{#1_{\type{ES}}}}
\newcommand{\Elimt}[1]{\ensuremath{#1_{\type{LIMT}}}}
\newcommand{\Vlim}[1]{\ensuremath{\var{#1}_{\type{LIM}}}}
\newcommand{\Vlimt}[1]{\Elimt{\var{#1}}}

\newcommand{\Eloc}[1]{\ensuremath{{#1}_{\type{LOC}}}}
\newcommand{\Vloc}[1]{\Eloc{\var{#1}}}
\newcommand{\Ves}[1]{\es{\var{#1}}}
\newcommand{\Vrule}[1]{\ensuremath{\var{#1}_{\type{RULE}}}}
\newcommand{\Vruleset}[1]{\ensuremath{\var{#1}_{\setOf{\type{RULE}}}}}
\newcommand{\Vsize}[1]{\ensuremath{\size{\var{#1}}}}
\newcommand{\Vstr}[1]{\ensuremath{\var{#1}_{\type{STR}}}}
\newcommand{\sym}[1]{#1_{\type{SYM}}}
\newcommand{\Vsym}[1]{\ensuremath{\var{#1}_{\type{SYM}}}}
\newcommand{\Vorig}[1]{\ensuremath{\var{#1}_{\type{ORIG}}}}
\newcommand{\Vsymset}[1]{\ensuremath{\var{#1}_{\setOf{\type{SYM}}}}}

\newcommand{\token}[1]{#1_{\type{TOK}}}

\newcommand{\alg}[1]{\ensuremath{\textsc{#1}}\xspace}
\newcommand{\AH}{\ensuremath{\alg{AH}}\xspace}
\newcommand{\Earley}{\ensuremath{\alg{Earley}}\xspace}
\newcommand{\Leo}{\ensuremath{\alg{Leo}}\xspace}
\newcommand{\Marpa}{\ensuremath{\alg{Marpa}}\xspace}

\newcommand{\Cfa}{\var{fa}}
\newcommand{\Cg}{\var{g}}
\newcommand{\Cw}{\var{w}}
\newcommand{\CVw}[1]{\ensuremath{\sym{\Cw[\var{#1}]}}}
\newcommand{\Crules}{\var{rules}}
\newcommand{\Nulling}[1]{\mymathop{Nulling}(#1)}
\newcommand{\Nullable}[1]{\mymathop{Nullable}(#1)}
\newcommand{\GOTO}{\mymathop{GOTO}}
\newcommand{\Next}[1]{\mymathop{Next}(#1)}
\newcommand{\Predict}[1]{\mymathop{Predict}(#1)}
\newcommand{\Postdot}[1]{\mymathop{Postdot}(#1)}
\newcommand{\Penult}[1]{\mymathop{Penult}(#1)}
\newcommand{\Pos}[1]{\mymathop{Pos}(#1)}
\newcommand{\Rule}[1]{\mymathop{Rule}(#1)}
\newcommand{\LHS}[1]{\mymathop{LHS}(#1)}
\newcommand{\RHS}[1]{\mymathop{RHS}(#1)}
\newcommand{\DR}[1]{\mymathop{DR}(#1)}
\newcommand{\Transition}[1]{\mymathop{Transition}(#1)}
\newcommand{\Origin}[1]{\mymathop{Origin}(#1)}
\newcommand{\RightRecursive}[1]{\mymathop{Right-Recursive}(#1)}
\newcommand{\RightNN}[1]{\mymathop{Right-NN}(#1)}
\newcommand{\LeoEligible}[1]{\mymathop{Leo-Eligible}(#1)}
\newcommand{\LeoUnique}[1]{\mymathop{Leo-Unique}(#1)}
\newcommand{\Confluences}[1]{\mymathop{Confluences}(#1)}
\newcommand{\LIMTMap}[1]{\mymathop{LIMT-Map}(#1)}
\newcommand{\LIMTMainstem}[1]{\mymathop{LIMT-Mainstem}(#1)}
\newcommand{\ID}[1]{\mymathop{ID}(#1)}
\newcommand{\PSL}[2]{\mymathop{PSL}[#1][#2]}
\newcommand{\myL}[1]{\mymathop{L}(#1)}
\newcommand\Etable[1]{\ensuremath{\mymathop{table}[#1]}}
\newcommand\bigEtable[1]{\ensuremath{\mymathop{table}\bigl[#1\bigr]}}

\newcommand\Rtablesize[1]{\ensuremath{\sum\seq{\bigl| \mymathop{table}[#1,\V{i}] \bigr|}}}
\newcommand\Vtable[1]{\Etable{\var{#1}}}
\newcommand\EEtable[2]{\ensuremath{\mymathop{table}[#1,#2]}}
\newcommand\EVtable[2]{\EEtable{#1}{\var{#2}}}

\newcommand\call[2]{\textproc{#1}\ifthenelse{\equal{#2}{}}{}{(#2)}}%

\newtheorem{theorem}{Theorem}[section]
\newtheorem{lemma}[theorem]{Lemma}

\theoremstyle{definition}
\newtheorem*{definition}{Definition}

\theoremstyle{remark}

\newtheorem{observation}[theorem]{Observation}

\newcommand{\firstRReqref}{UNDEFINED!}
\newcommand{\lastRReqref}{UNDEFINED!}

\newenvironment{MYsloppy}[1][3em]{\par\tolerance9999 \emergencystretch#1\relax}{\par}
\newenvironment{MYsloppyB}[2]{\par\tolerance9999 \emergencystretch#1 \hbadness#2\relax}{\par}
\newenvironment{MYsloppyC}[3]{\par\tolerance#3 \emergencystretch#1 \hbadness#2\relax}{\par}

\begin{document}

\date{\today.  Revision 3.}

\title{Marpa, a practical general parser: the recognizer}

\author{Jeffrey Kegler}
\thanks{%
Copyright \copyright\ 2022 Jeffrey Kegler.
This document is licensed under
a Creative Commons Attribution-NoDerivs 3.0 United States License.
}

\begin{abstract}
The \Marpa recognizer is described.
Marpa is
a practical and fully implemented
algorithm for the recognition,
parsing and evaluation of context-free grammars.
The \Marpa recognizer is the first
to unite the improvements
to Earley's algorithm found in
Joop Leo's 1991 paper
to those in Aycock and Horspool's 2002 paper.
Marpa tracks the full state of the parse,
as it proceeds,
in a form convenient for the application.
This greatly improves error detection
and enables event-driven parsing.
One such technique is
``Ruby Slippers'' parsing,
in which
the input is altered in response
to the parser's expectations.
\end{abstract}

\maketitle

\section{Introduction}

Despite the promise of general context-free parsing,
and the strong academic literature behind it,
as of 2010
it had never been incorporated into a highly available tool
like those that exist for LALR~\cite{Johnson} or
regular expressions.
The \Marpa project was intended
to take the best results from the literature
on Earley parsing off the pages
of the journals and bring them
to a wider audience.
Marpa::XS~\cite{Marpa-XS},
a stable version of this tool,
was uploaded to the CPAN Perl archive
on Solstice Day in 2011.
This paper describes the algorithm of Marpa::R2~\cite{Marpa-R2},
a later version.

As implemented in~\cite{Marpa-R2},
Marpa parses,
without exception,
all context-free grammars.
Time bounds are the best of Leo~\cite{Leo1991}
and Earley~\cite{Earley1970}.
The Leo bound,
\On{} for LR-regular grammars,
is especially relevant to
Marpa's goal of being a practical parser:
If a grammar is in a class of grammar currently in practical use,
Marpa parses it in linear time.

\begin{MYsloppy}
Error-detection properties are
extremely important for practical parsing,
but have been overlooked in the past.
Marpa breaks new ground in this respect.
Marpa has the immediate error detection property,
and goes well beyond that:
it is fully aware of the state of the parse,
and can report this to the user while tokens are
being scanned.
\end{MYsloppy}

Marpa allows the lexer to check its list
of acceptable tokens before a token is scanned.
Because rejection of tokens is easily and
efficiently recoverable,
the lexer is also free to take an event-driven
approach.
Error detection is no longer
an act of desperation,
but a parsing technique in its own right.
If a token is rejected,
the lexer is free to create a new token
in the light of the parser's expectations.
This approach can be described
as making the parser's
``wishes'' come true,
and we have called this
``Ruby Slippers Parsing''.

One use of the Ruby Slippers technique is to
parse with a clean
but oversimplified grammar,
programming the lexical analyzer to make up for the grammar's
short-comings on the fly.
The author has implemented an HTML parser~\cite{Marpa-HTML},
based on a grammar that assumes that all start
and end tags are present.
Such an HTML grammar is too simple even to describe perfectly
standard-conformant HTML,
but the lexical analyzer is
programmed to supply start and end tags as requested by the parser.
The result is a very simply and cleanly designed parser
that parses very liberal HTML
and accepts all input files,
in the worst case
treating them as highly defective HTML.

Section
\ref{s:preliminaries} describes the notation and conventions
of this paper.
Section \ref{s:rewrite} deals with \Marpa's
grammar rewrites.
Sections \ref{s:earley}, \ref{s:confluences},
and \ref{s:earley-ops}
introduce Earley's algorithm.
Section \ref{s:leo} describes Leo's modification
to Earley's algorithm.
Section \ref{s:AHFA} describes the modifications
proposed by Aycock and Horspool.
Section \ref{s:pseudocode} presents the pseudocode
for \Marpa's recognizer.
Section
\ref{s:proof-preliminaries}
describes notation for,
and other preliminaries to,
the theoretical results.
Section
\ref{s:correct}
contain a proof of \Marpa's correctness,
while Section \ref{s:complexity} contains
its complexity results.
Finally,
Section \ref{s:input}
generalizes \Marpa's input model.

The nature of this paper is such that
an adequate literature survey
would be as large as the rest of this paper.
Instead, we have placed
a full, if somewhat informal, literature survey
online~\cite{Timeline}.

\section{Preliminaries}
\label{s:preliminaries}
\label{s:start-prelim}

We assume familiarity with the theory of parsing,
as well as Earley's algorithm.
We will use the type system
of Farmer~2012~\cite{Farmer2012},
without needing most of its apparatus.
The notation \VTtyped{x}{T} indicates that the variable
\V{x} is of type \V{T}.
More often, this paper will
use subscripts to indicate type.
We will often designate particular sets as types,
but any set can be a type.\footnote{
In fact,
a type in Farmer~2012~\cite{Farmer2012} can be much
more than a set.
Types in \cite{Farmer2012} are
collections of classes (``superclasses'').
But in this paper, every explicitly stated type will be a ZF set.
}
\begin{center}
\begin{tabular}{ll}
$\VTtyped{X}{T}$ & The variable \var{X} of type \type{T} (wide form)\\
$\var{X}_\type{T}$ & The variable \var{X} of type \type{T} (narrow form)\\
$\var{set-one}_\setOf{\type{T}}$ & The variable \var{set-one} of type set of \type{T} \\
$\type{SYM}$ & The type for a symbol \\
\Vsym{a} & The variable \var{a} of type $\type{SYM}$ \\
\Vsymset{set-two} & The variable \var{set-two}, a set of symbols \\
\end{tabular}
\end{center}
Subscripts may be omitted when the type
is obvious from the context.
Multi-character variable names will be common,
and operations will never be implicit.
\begin{center}
\begin{tabular}{ll}
Multiplication &  $\var{a} \mult \var{b}$ \\
Concatenation & $\var{a} \cat \var{b}$ \\
Subtraction & $\var{symbol-count} \subtract \var{terminal-count}$ \\
\end{tabular}
\end{center}

We often write ``iff'' for ``if and only if''.
We also often subsitute the more prominent double colon ($\mcoloncolon$)
for the ``mid'' divider ($\mid$).

A useful feature of Farmer~2012~\cite{Farmer2012} is his notion of ill-definedness.
For example, the value of partial functions may be ill-defined for
some arguments in their domain.
Farmer's handling of ill-definedness is the traditional one,
and was well-entrenched,
but he was first to describe and formalize it.\footnote{%
See Farmer 2004~\cite{Farmer2004}.
Note that Farmer refers to ill-defined values as ``undefined''.
We found this problematic.
For example, a partial function may not have a value for
every argument in its domain.
Saying that that the value of the partial function for these arguments
is defined as ``undefined'' is confusing.
In this paper we say that these values are defined, but ill-defined.
}

We write $\unicorn$ for ill-defined.
A value is \dfn{well-defined} iff it is not ill-defined.
We write $\V{x}\isWellDef$ to say that \V{x} is well-defined,
and we write $\V{x}\isIllDef$ to say that \V{x} is ill-defined.

Traditionally, and in this paper,
any formula with an ill-defined operand is false.
This means that a equality both of whose operands are ill-defined
is false, so that $\neg \; (\unicorn = \unicorn)$.
For cases where this is inconvenient,
we introduce a new relation, $\simeq$, such that
$$
  \V{a} \simeq \V{b} \defined \V{a} = \V{b} \lor (\V{a} \isIllDef \land \V{b} \isIllDef ).
$$

\Needspace{5\baselineskip}
We define a tuple recursively:
\begin{itemize}
\item An ordered pair is a 2-tuple, or duple,
for example $\tuple{7,11}$.
The first entry of a 2-tuple is its \dfn{head}.
The second entry of a 2-tuple is its \dfn{tail}.
\item The ordered pair of a set \V{h},
and an \V{n}-tuple, call it \V{tupA},
is an ($\V{n}+1$)-tuple, call it \V{tupB}.
\V{h} is the \dfn{head} of \V{tupB}.
\V{tupA} is the \dfn{tail} of \V{tupB}.
\end{itemize}

We write tuples using angle brackets.
The head of a tuple \V{S} can be written \Hd{S}.
The tail of a tuple \V{S} can be written \Tl{S}.
For example, where
\[
    \V{S} = \tuple{ 42, 1729, 42 },
\]
then \V{S} is a 3-tuple, or triple,
\begin{gather*}
\Hd{S}=42, \\
\Tl{S}=\tuple{1729,42}, \\
\HTl{S}=1729, \\
\TTl{S}=42, \text{ and} \\
\TTTl{S}\isIllDef.
\end{gather*}

We define the natural numbers, \naturals, to include zero.
$\V{D} \fnMap \V{C}$ is the set of partial functions from domain \V{D} to codomain \V{C}.
$\V{D} \tfnMap \V{C}$ is the set of total functions from domain \V{D} to codomain \V{C}.
It follows that $\naturals \tfnMap \V{C}$ is the set of
infinite sequences of terms from the set \V{C};
and that $42 \tfnMap \V{C}$ is the set of
sequences of length 42 of terms from the set \V{C}.
We say that
$$ \V{domSet} \tfnMMap \V{C} \defined
   \bigl\{ \V{fn} \bigm|
       \left( \exists \; \V{D} \in \V{domSet} \; \middle| \; \V{fn} \in \V{D} \tfnMap \V{C}
       \right) \bigr\}
$$
so that $\naturals \tfnMMap \V{C}$
is the set of finite sequences of terms from the set \V{C}.

We write \Vsize{\V{seq}} for the cardinality, or length,
of the sequence \V{seq}.
\VVelement{seq}{i} is the \var{i}'th term of the sequence \V{seq},
and is well-defined when
$0 \le \V{i} < \Vsize{seq}$.
We write often specify a sequence by giving its terms inside
brackets.
For example,
$$\seq{\, \V{a}, \; 42, \; \seq{} \, }$$
is the sequence of length
3 whose terms are, in order, \V{a}, the number 42, and the empty sequence.

The last index of the sequence \V{seq} is \Vlastix{seq},
so that \VlastElement{seq} is the last term of \V{seq}.
\Vlastix{seq} is ill-defined for the empty sequence,
that is, if $\Vsize{seq} = 0$.
If \V{seq} is not the empty sequence, then
$\Vsize{seq} = \Vlastix{seq} + 1$.

To avoid sub- and superscripts,
we usually write summation as a unary operation on a sequence.
For example, we write
$$ \sum \seqB{\Vtyped{ix}{\naturals} < 100}{\V{ix}} \text{ for }
     \sum\limits_{\V{ix}_\naturals=0}^{100}\V{ix}.
$$

Let \V{vocab} be a non-empty set of symbols.
The string type, \type{STR}, is the set of all finite sequences of
symbols:
$$ \type{STR} = \naturals \tfnMMap \V{vocab} $$

Where \VTtyped{s}{STR} is a string,
we write \size{\Vstr{s}} for the string length, counted in symbols.
We write $\type{STR}^+$ for the set of all non-null strings:
$$
\type{STR}^+ =
\bigl\{ \VTtyped{x}{STR} \bigm| \Vsize{\Vstr{x}} > 0 \bigr\}.
$$

In this paper we use,
without loss of generality,
the grammar \Cg{},
where \Cg{} is the 3-tuple
\begin{equation*}
    \tuple{\Vsymset{vocab}, \var{rules}, \Vsym{accept}}.
\end{equation*}
Here $\Vsym{accept} \in \var{vocab}$.
Call the language of \V{g}, $\typed{\myL{\Cg}}{\setOf{\type{STR}}}$.

\Vruleset{rules} is a set of rules (type \type{RULE}),
where a rule is a duple
of the form $\tuple{\Vsym{lhs} \de \Vstr{rhs}}$,
such that
\begin{equation*}
\Vsym{lhs} \in \var{vocab} \quad \text{and}
\quad \Vstr{rhs} \in \type{STR}^+.
\end{equation*}
\Vsym{lhs} is referred to as the left hand side (LHS)
of \Vrule{r}.
\Vstr{rhs} is referred to as the right hand side (RHS)
of \Vrule{r}.
The LHS and RHS of \Vrule{r} may also be
referred to as
$\LHS{\Vrule{r}}$ and $\RHS{\Vrule{r}}$, respectively.
This definition follows~\cite{AH2002},
which departs from tradition by disallowing an empty RHS.

The rules imply the traditional rewriting system,
in which
\begin{itemize}
\item $\Vstr{x} \derives \Vstr{y}$
states that \var{x} derives \var{y} in exactly one step;
\item $\Vstr{x} \deplus \Vstr{y}$
states that \var{x} derives \var{y} in one or more steps;
and
\item $\Vstr{x} \destar \Vstr{y}$
states that \var{x} derives \var{y} in zero or more steps.
\end{itemize}
We call these rewrites \dfn{derivation steps}.
A sequence of one or more of derivation steps,
in which the left hand side of all but the first
is the right hand side of its predecessor,
is a \dfn{derivation}.
We say that the symbol \V{x} \dfn{induces}
the string of length 1 whose only term is that symbol,
that is, the string \seq{\Vsym{x}}.
Pedantically, the terms of derivations,
and the arguments of concatenations like
$\Vstr{s1} \cat \Vstr{s2}$
must be strings.
But in concatenations and derivation steps we often
write the symbol to represent the string it induces so
that
$$ \Vstr{a} \cat \Vsym{b} \cat \Vstr{c}
= \Vstr{a} \cat \seq{\Vsym{b}} \cat \Vstr{c}.
$$

We say that a string \V{x} is \dfn{nullable},
$\Nullable{\Vstr{x}}$, iff the empty string can be
derived from it:
$$\Nullable{\Vstr{x}} \defined \V{x} \destar \epsilon.$$
We say that a string \V{x} is \dfn{nulling},
$\Nulling{\Vstr{x}}$,
iff it always eventually derives the null
string:
\[
\Nulling{\Vstr{s}} \defined
  \forall \; \Vstr{y} \;\mid\; \V{x} \destar \V{y} \implies \V{y} \destar \epsilon.
\]
We say that symbols are nulling or nullable based on the string they
induce:
\begin{gather*}
\Nullable{\Vsym{x}} \defined \Nullable{\seq{\V{x}}}. \\
\Nulling{\Vsym{x}} \defined \Nulling{\seq{\V{x}}}.
\end{gather*}
A string or symbol is
\begin{itemize}
\item \dfn{non-nullable} iff it is not nullable;
\item a \dfn{proper nullable} iff it is nullable,
but not nulling; and
\item \dfn{non-nulling} iff it is not nulling.
\end{itemize}

Following Aycock and Horspool~\cite{AH2002},
all nullable symbols in grammar \Cg{} are nulling -- every symbol
which can derive the null string always derives the null string.
It is shown in~\cite{AH2002} how to do this without losing generality
or the ability to efficiently evaluate a semantics that is
defined in terms of an original grammar that includes symbols which
are both nullable and non-nulling,
empty rules, etc.

Also without loss of generality,
it is assumed
that there is a dedicated acceptance rule,
$$ \Vrule{accept} = \tuple{ \Vsym{accept} \de \Vsym{start} }, $$
where $\Vsym{start} \in \var{vocab}$,
and that the accept symbol, \Vsym{accept},
is such that
\begin{equation*}
\begin{gathered}
\forall \; \V{x} \in \Crules \mid \\
\nexists \; \Vstr{pre}, \; \Vstr{post} \; \mid \\
\V{pre} \cat \Vsym{accept} \cat \V{post} = \RHS{\Vrule{x}} \\
\end{gathered}
\end{equation*}
and
$$ \Vsym{accept} = \LHS{\V{x}} \implies \Vrule{accept} = \V{x}. $$

\begin{MYsloppy}
Our definition of the rightmost non-nulling symbol of a string,
$\RightNN{\Vstr{x}}$,
is
\end{MYsloppy}
\[\begin{gathered}
\RightNN{\Vstr{x}} \defined \; \iotaQ{ \VTtyped{rnn}{SYM} }{ \\
    \existQ{ \VTtyped{pre}{STR}, \; \VTtyped{post}{STR} \;}{ \\
      \Vstr{x} = \V{pre} \cat \V{rnn} \cat \V{post} \land \Nullable{\V{post}} \land \neg \Nulling{\V{rnn}}.
    } }
\end{gathered}\]
\begin{MYsloppy}
Our definition of the rightmost non-nulling symbol of a rule,
$\RightNN{\Vrule{r}}$, is
\end{MYsloppy}
\begin{equation*}
\RightNN{\Vrule{r}} \defined \RightNN{\RHS{\V{r}}}.
\end{equation*}
A rule \Vrule{x} is \dfn{directly right-recursive}
if and only if
\begin{equation*}
\LHS{\Vrule{x}} = \RightNN{\Vrule{x}}.
\end{equation*}
\Vrule{x} is \dfn{right-recursive},
$\RightRecursive{\Vrule{x}}$,
if and only if
\begin{equation*}
\exists \, \Vstr{y} \mid \RightNN{\Vrule{x}} \deplus \V{y} \land \RightNN{\V{y}} = \LHS{\Vrule{x}}.
\end{equation*}

Our definition of a grammar did not sharply distinguish terminals
from non-terminals.
The implementation of~\cite{Marpa-R2} allows terminals to be the LHS
of rules,
and every symbol except \Vsym{accept} can be a terminal.
\cite{Marpa-R2} has options that allow
the user to reinstate
the traditional restrictions,
in part or in whole.
We note that,
as a result of these definitions,
sentential forms will be of type \type{STR}.

Let the input to
the parse be $\typed{\Cw}{\type{STR}^+}$.
Locations in the input will be of type \type{LOC}.
When we state our complexity results later,
they will often be in terms of $\var{n}$,
where $\var{n} = \Vsize{w}$.
\CVw{i} is character \var{i}
of the input,
and is well-defined when
$0 \le \Vloc{i} < \Vsize{w}$.

We note that the previous definition
of \Cw{} did not allow zero-length inputs.
To simplify the mathematics, we exclude null parses
and trivial grammars from consideration.
In the implementation of~\cite{Marpa-R2},
the \Marpa parser
deals with null parses and trivial grammars as special cases.
(Trivial grammars are those that recognize only the null string.)

In this paper,
\Earley will refer to the Earley algorithm as it is presented in this paper
--- a simplified version of Earley~1970~\cite{Earley1970}
preceded and followed by a rewrite.
\Leo will refer to Leo's revision of \cite{Earley1970}
as described in Leo 1991~\cite{Leo1991}.
\AH will refer to the Aycock and Horspool's revision
of \cite{Earley1970}
as described in their 2002 paper~\cite{AH2002}.
\Marpa will refer to the parser described in
this paper,
which combines features of \Earley, \Leo and \AH.
Where $\alg{Recce}$ is a recognizer,
$\myL{\alg{Recce},\Cg}$ will be the language accepted by $\alg{Recce}$
when parsing \Cg{}.

\section{Rewriting the grammar}
\label{s:rewrite}

We have already noted
that no rules of \Cg{}
have a zero-length RHS,
and that all symbols must be either nulling or non-nullable.
These restrictions follow Aycock and Horspool~\cite{AH2002}.
The elimination of empty rules and proper nullables
is done by rewriting the grammar.
\cite{AH2002} shows how to do this
without loss of generality.

Because \Marpa claims to be a practical parser,
it is important to emphasize
that all grammar rewrites in this paper
are done in such a way that the semantics
of the original grammar can be reconstructed
simply and efficiently at evaluation time.
As one example,
when a rewrite involves the introduction of new rule,
semantics for the new rule can be defined to pass its operands
up to a parent rule as a list.
Where needed, the original semantics
of a pre-existing parent rule can
be ``wrapped'' to reassemble these lists
into operands that are properly formed
for that original semantics.

In the implementation of~\cite{Marpa-R2},
the \Marpa parser allows users to associate
semantics with an original grammar
that has none of the restrictions imposed
on grammars in this paper.
The user of a \Marpa parser
may specify any context-free grammar,
including one with properly nullable symbols,
empty rules, etc.
The user specifies his semantics in terms
of this original, ``free-form'', grammar.
\cite{Marpa-R2} implements the rewrites,
and performs evaluation,
in such a way as to keep them invisible to
the user.
From the user's point of view,
the ``free-form'' of his grammar is the
one being used for the parse,
and the one to which
his semantics are applied.

\section{Earley's algorithm}
\label{s:earley}

This paper presents a specialized version of Earley's
algorithm.
The version in this paper is simplified to
take advantage of \Marpa{}'s rewriting.
Descriptions of
the standard \Earley{}'s algorithm
are now plentiful.%
\footnote{%
    Focusing on the classic Earley-related literature,
    these include
    \cite[pp. 320-321]{AU1972},
    \cite{AH2002},
    \cite{Earley1968},
    \cite{Earley1970},
    \cite{GJ2008}, and
    \cite{Leo1991}.%
}

A dotted rule (type \type{DR}) is a duple of rule and position
in the rule.
\begin{gather*}
\Rule{\Vdr{dr}} \defined \Hd{\V{dr}}, \\
\Pos{\Vdr{dr}} \defined \Tl{\V{dr}}, \text{ and} \\
 0 \le \Pos{\V{dr}} \le \size{\Rule{\V{dr}}}.
\end{gather*}

The position of a dotted rule indicates the extent to which
the rule has been recognized,
and is represented with a large raised dot,
so that if
\begin{equation*}
\tuple{ \Vsym{A} \de \Vsym{X} \cat \Vsym{Y} \cat \Vsym{Z}}
\end{equation*}
is a rule,
\begin{equation*}
\tuple{ \Vsym{A} \de \var{X} \cat \var{Y} \mydot \var{Z}}
\end{equation*}
is the dotted rule with the dot at
$\var{pos} = 2$,
between \Vsym{Y} and \Vsym{Z}.

Every rule concept, when applied to a dotted rule,
is applied to the rule of the dotted rule.
The following are examples:
\[
\begin{gathered}
  \LHS{\Vdr{x}} \defined \LHS{\Rule{\V{x}}}. \\
  \RHS{\Vdr{x}} \defined \RHS{\Rule{\V{x}}}. \\
  \RightRecursive{\Vdr{x}} \defined \RightRecursive{\Rule{\V{x}}}.
\end{gathered}
\]

We also state the following:
\[
  \Postdot{\Vdr{x}} \defined
  \begin{cases}
  \Vsym{next}, \text{ if } \existQ{
      \Vstr{pre}, \; \Vstr{post}, \; \Vsym{A}}{ \\
      \qquad \V{x} = \tuple{ \V{A} \de \V{pre} \mydot \V{next} \cat \V{post}}} \\
  \unicorn, \text{ otherwise.}
  \end{cases} \\
\]

\[
  \Next{\Vdr{x}} \defined
  \begin{cases}
    \VTtyped{\tuple{ \Vsym{A} \de \Vstr{pre} \cat \Vsym{next} \mydot \Vstr{post}}}{DR}, \\
      \qquad \text{if } \V{x} = \tuple{ \V{A} \de \V{pre} \mydot \V{next} \cat \V{post}} \\
    \unicorn, \text{ otherwise.}
  \end{cases} \\
\]

\[
  \Penult{\Vdr{x}} \defined
  \begin{cases}
  \Vsym{next}, \text{ if } \existQ{
      \Vstr{pre}, \; \Vstr{post}, \; \Vsym{A}}{ \\
      \qquad \V{x} = \tuple{ \V{A} \de \V{pre} \mydot \V{next} \cat \V{post}} \\
      \qquad \land \; \Nullable{\V{post}} \land \; \neg \Nulling{\V{next}}} \\
  \unicorn, \text{ otherwise}
  \end{cases}
\]

A \dfn{penult} is a dotted rule \Vdr{d} such that $\Penult{\var{d}} \isWellDef$.
We note that $\Penult{\Vdr{x}}$
is never a nullable symbol.
The \dfn{initial dotted rule} is
\begin{equation}
\label{eq:initial-dr}
\Vdr{initial} = \tuple{\Vsym{accept} \de \mydot \Vsym{start} }.
\end{equation}
A \dfn{predicted dotted rule} is a dotted rule,
other than the initial dotted rule,
with a dot position of zero,
for example,
\begin{equation*}
\Vdr{predicted} = \tuple{\Vsym{A} \de \mydot \Vstr{alpha} }.
\end{equation*}
A \dfn{confirmed dotted rule}
is the initial dotted rule,
or a dotted rule
with a dot position greater than zero.
A \dfn{completed dotted rule} is a dotted rule with its dot
position after the end of its RHS,
for example,
\begin{equation*}
\Vdr{completed} = \tuple{\Vsym{A} \de \Vstr{alpha} \mydot }.
\end{equation*}
Predicted, confirmed and completed dotted rules
are also called, respectively,
\dfn{predictions}, \dfn{confirmations} and \dfn{completions}.

A traditional Earley item (type \type{EIMT}) is a duple
of dotted rule and origin.
\begin{gather*}
     \DR{\Veimt{x}} = \Hd{\V{x}}. \\
     \Origin{\Veimt{x}} = \Tl{\V{x}}.
\end{gather*}
The origin is the location where recognition of the rule
started.
For convenience, the type \type{ORIG} will be a synonym
for \type{LOC}, indicating that the variable designates
the origin entry of an Earley item.

We find it convenient to apply dotted rule concepts to
\type{EIMT}'s,
so that the concept applied to the \type{EIMT} is
the concept applied to the dotted rule of the \type{EIMT}.
The following are examples:
\begin{gather*}
\LHS{\Veimt{x}} \defined \LHS{\DR{\V{dr}}}. \\
\Postdot{\Veimt{x}} \defined \Postdot{\DR{\V{dr}}}. \\
\Penult{\Veimt{x}} \defined \Penult{\DR{\V{dr}}}. \\
\Next{\Veimt{x}} \defined \Next{\DR{\V{dr}}}. \\
\Rule{\Veimt{x}} \defined \Rule{\DR{\V{dr}}}. \\
\RightRecursive{\Veimt{x}} \defined \RightRecursive{\DR{\V{dr}}}.
\end{gather*}

An Earley parser builds a table of Earley sets,
\begin{equation*}
\EVtable{\Earley}{i},
\quad \text{where} \quad
0 \le \Vloc{i} \le \size{\Cw}.
\end{equation*}
Earley sets are of type \type{ES}.
Earley sets are often named by their location:
That is,
the bijection between \Ves{i} and \Vloc{i}
allows locations to be treated as the ``names'' of Earley sets.
We often write \Ves{i} to mean the Earley set at \Vloc{i},
and \Vloc{x} to mean the location of Earley set \Ves{x}.
The type designator \type{ES} is often omitted to avoid clutter,
especially in cases where the Earley set is not
named by location.
Occasionally the naming location is a expression,
so that
$$ \es{(\Origin{\Veimt{x}})} $$
is the Earley set at the origin of the \type{EIMT} \Veimt{x}.

At points,
we will need to compare the Earley sets
produced by the different recognizers.
\EVtable{\alg{Recce}}{i} will be
the Earley set at \Vloc{i}
in the table of Earley sets of
the \alg{Recce} recognizer.
For example,
\EVtable{\Marpa}{j} will be Earley set \Vloc{j}
in \Marpa's table of Earley sets.
In contexts where it is clear which recognizer is
intended,
\Vtable{k}, or \Ves{k}, will symbolize Earley set \Vloc{k}
in that recognizer's table of Earley sets.
If \es{\var{working}} is an Earley set,
$\size{\es{\var{working}}}$ is the number of Earley items
in \es{\var{working}}.

We often want the count of all the Earley items
in a table,
and we abbreviate the expression for this by omitting
the quantifier, so that
\begin{equation}
\label{eq:abbr-all-eims-count}
\begin{gathered}
\Rtablesize{\alg{Recce}} \defined \\
  \sum\seqB{\Vtyped{i}{\naturals} \le \Vsize{w}}{\bigl| \mymathop{table}[\alg{Recce},\V{i}] \bigr|}.
\end{gathered}
\end{equation}
For example,
\Rtablesize{\Marpa} is the total number
of Earley items in all the Earley sets of
a \Marpa{} parse.

Recall that
there was a unique acceptance symbol,
\Vsym{accept}, in \Cg{}.
The input \Cw{} is accepted if and only if,
\begin{equation*}
\tuple{\tuple{\Vsym{accept} \de \Vsym{start} \mydot}, 0} \in \bigEtable{\Vsize{\Cw}}
\end{equation*}

\section{Confluences}
\label{s:confluences}

An Earley item is also called a \dfn{parse item}.
For \Marpa, we will define another kind of parse item,
a Leo item, later.
Traditional parse items have type \type{PIMT}.

In \Marpa, with each parse item is a set of
confluences, which
track the reasons the \Marpa algorithm had
for adding that parse to the
Earley set.
In an ambiguous parse, the \Marpa algorithm may have
more than one reason to add a parse item to an Earley set.
Each reason is called a \dfn{confluence}.

A \dfn{confluence} is a duple,
whose entries are called \dfn{inflows}.
The first inflow of a confluence is the \dfn{mainstem},
and is either a parse item or ill-defined.
The second inflow of a confluence is the \dfn{tributary},
and may be an Earley item, a token, or ill-defined.
The hydrological terminology may seem ornate,
but experience
has shown that the overloading of more ordinary terms like ''predecessor'',
``cause'', and ``component''\footnote{%
The term ``component'' was used in Irons~\cite{Irons}.
}
can be befuddling.

We hope this terminology is at least mildly intuitive.
In hydrology,
a confluence is a meeting of two upstream river channels to form a third, downstream, channel.
Of the two upstream channels, one (usually the larger) is the mainstem,
and the other is a tributary.
For example, near Cairo, Illinois, there is a confluence of
the Ohio and Mississippi Rivers,
in which the the upstream Mississippi channel
is the mainstem
and the Ohio is a tributary.

Continuing the hydrological analogy,
a sequence of parse items in which all but the
first is the successor of its mainstem
is called a \dfn{trunk}.
The first term in the sequence is the \dfn{top} of the \dfn{trunk},
and the last term is the \dfn{bottom} of the trunk.
Sequence terms which are neither top or bottom are
\dfn{interior} terms of the trunk.

Similarly, we can define a \dfn{tributary sequence}
as a sequence in which all but the first is the successor of
its tributary.
Once again,
the first term in the tributary
sequence is the \dfn{top} of the \dfn{trunk},
and the last term is the \dfn{bottom} of the trunk.
Tributary sequence terms which are neither top or bottom are
\dfn{interior} terms of the trunk.

We sometimes refer to the confluences and their inflows
as the \dfn{links} of parse items,
reflecting that fact that they are typically implemented
as pointers, or ``links''.
We also sometimes refer to confluences as
\dfn{causations} of a parse item,
because each confluence is the reason for
the Earley algorithm to add the parse item
to the Earley set.

\section{Operations of the Earley algorithm}
\label{s:earley-ops}

\begin{MYsloppyB}{3em}{3000}
For any Earley operation there is a
current parse location, \Vloc{current}, and
a current Earley set, \Ves{current}.
Recall that
we often write \Ves{i} for
the Earley set at \Vloc{i},
and \es{(\var{exp})} for the Earley set
at the location given by the expression \var{exp}.
\end{MYsloppyB}

We write the set of confluences of a PIMT
\Vpimt{p} in Earley set \V{es} as
$\Confluences{\V{p},\V{es}}.$
Each Earley operation is shown in the form of
an inference rule,
the conclusion of which
consists of
\begin{itemize}
\item a parse item, call it \Vpimt{p} of the Earley items to be added to \Ves{current}; and
\item a confluence to be added to $\Confluences{\Vpimt{p},\Ves{current}}$.
\end{itemize}
We note that when we said the confluence and parse item were ``added'' just now,
that they are added to sets, and that an ``add'' is a no-op for an object that is
already an element of the set.
Implementations must take care not to allow duplicate confluences in confluence sets,
and not to allow duplicate Earley items in Earley sets.

Each location starts with an empty Earley set.
For the purposes of this description of
\Earley{}, the order of the
Earley operations
when building an
Earley set is non-deterministic.
When no more Earley items
can be added, the Earley set is
complete.
In the \Marpa implementation,
the Earley sets are built in
order from 0 to \Vsize{w}.

\subsection{Initialization}
\label{d:initial}
\begin{equation*}
\inference{
   \Vloc{current} = 0
}{
    \begin{array}{c}
        \VTtyped{ \tuple{ \Vdr{initial}, 0 }}{EIMT} \\
	\VTtyped{\tuple{\unicorn, \unicorn}}{CFLU}
    \end{array}
}
\end{equation*}
Here \Vdr{initial} is from \eqref{eq:initial-dr}.
Earley {\bf initialization} only takes
place in Earley set 0,
and always adds exactly one \type{EIM},
with exactly one confluence.
The mainstem and tributary of the confluence are both ill-defined.

\subsection{Scanning}
\label{d:scan}
\begin{equation*}
\inference{
    \begin{array}{c}
	\Veimt{mainstem} \in \es{ ( \Vloc{current} \subtract 1 ) } \\
	\Postdot{\Veimt{mainstem}} =
	   \Cw\bigl[ \Vloc{current} \subtract 1 \bigr]
    \end{array}
}{
    \begin{array}{c}
	\VTtyped{ \tuple{ \Next{\Veimt{mainstem}}, \Origin{\var{mainstem}} }}{EIMT} \\
	\VTtyped{\tuple{\Veim{mainstem}, (
	   \Cw\bigl[ \Vloc{current} \subtract 1 \bigr]_\type{SYM}
	) }}{CFLU}
    \end{array}
}
\end{equation*}

In the confluence added by a scanning operation,
\Veimt{mainstem} is the mainstem,
and the symbol
$\Cw[ \Vloc{current} \subtract 1 ]$
is the tributary.
In the context of a parse location,
a symbol is often called a \dfn{token}.
We also say that
$\Postdot{\Veimt{mainstem}}$ is the transition symbol
of the confluence, and of the scanning operation.

\subsection{Reduction}
\label{s:reduction}

\begin{equation*}
\inference{
    \begin{array}{c}
    \Veimt{tributary} \in \Ves{current} \\
    \Veimt{mainstem} \in \es{(\Origin{\Veimt{tributary}})} \\
    \Postdot{\Veimt{mainstem}} = \LHS{\Veimt{tributary}} \\
    \end{array}
}{
    \begin{array}{c}
	\VTtyped{ \tuple{ \Next{\Veimt{mainstem}}, \Origin{\var{mainstem}} }}{EIMT} \\
	\VTtyped{\tuple{\Veim{mainstem}, \Veim{tributary} }}{CFLU}
    \end{array}
}
\end{equation*}

$\Postdot{\Veimt{tributary}}$
is the transition symbol
of the reduction operation.

\subsection{Prediction}
\label{d:prediction}
\begin{equation*}
\inference{
    \begin{array}{c}
	\Veimt{mainstem} \in \Ves{current} \\
	\Postdot{\V{mainstem}} = \Vsym{lhs} \\
	\tuple{ \Vsym{lhs} \de \Vstr{rhs} } \in \Crules  \\
    \end{array}
}{
  \begin{array}{c}
      \VTtyped{ \tuple{ \tuple{ \Vsym{lhs} \de \mydot \Vstr{rhs} }, \Vloc{current} }}{EIMT} \\
      \VTtyped{\tuple{\Veimt{mainstem}, \unicorn }}{CFLU}
  \end{array}
}
\end{equation*}
The \type{EIMT} added by a prediction operation can be the mainstem
of other prediction operations at \Vloc{current},
so that one prediction operation can trigger a long series of others.
These prediction operations can add many Earley items
to \Ves{current},
but the items added will not depend on the location or the input ---
they will depend only on the postdot symbol of the mainstem.
This means that a number of optimizations are possible.

The tributary of the operation is ill-defined.

\section{The Leo algorithm}
\label{s:leo}

In~\cite{Leo1991}, Joop Leo presented a method for
dealing with right recursion in \On{} time.
Leo showed that,
with his modification, Earley's algorithm
is \On{} for all LR-regular grammars.
(LR-regular is LR where lookahead
is infinite length, but restricted to
distinguishing between regular expressions.)

Summarizing Leo's method,
it consists of spotting potential right recursions
and memoizing them.
Leo restricts the memoization to situations where
the right recursion is unambiguous.
Potential right recursions are memoized by
Earley set, using what Leo called
``transitive items''.
In this paper Leo's ``transitive items''
will be called Leo items.
Leo items in the form that \Marpa uses
will be type \type{LIM}.
``Traditional'' Leo items,
that is, those of the form used in Leo's paper~\cite{Leo1991},
will be type \type{LIMT}.

We illustrate the Leo method by examining an \Earley parse
of the input ``{\tt xxxx}''
using the grammar
\begin{equation}
\label{eq:leo-example-grammar}
\begin{aligned}
  &\tuple{\V{S} \de \V{RR}} \\
  &\tuple{\V{RR} \de \V{x}} \\
  &\tuple{\V{RR} \de \V{x} \cat \V{RR}}.
\end{aligned}
\end{equation}
This grammar and input produce the Earley sets
in the display which follows.
Each line in the display shows one \type{EIMT},
and has four columns.
\begin{itemize}
\item The first column shows the \type{EIMT}.
The notation $\text{@}\V{n}$ indicates that the \type{EIMT}
in \type{ES} $\V{n}$,
so that
$\text{@}3$ indicates that EIMT which precedes it
is in \type{ES} 3.
\end{itemize}
Since our primary interest is in completions,
the next three columns are shown only for completed
\type{EIMT}'s.
\begin{itemize}
\item
The second column is a note.
``Accept'' indicates that the \type{EIMT} is an accept \type{EIMT}.
``Bottom'', ``Interior'' and ``Top'' indicate that the \type{EIMT}
is in a corresponding position in a Leo stack.
Leo stacks will be explained below.
\item
The third column shows the \type{EIMT}'s confluence.
Our grammar is unambiguous,
so there is always exactly one confluence for each \type{EIMT}.
Within the confluence, \type{EIMT} inflows are represented by their
equation number.
\item
The fourth column indicates whether the EIM is actually used as part
of the parse.
Our grammar is unambiguous,
so there is only one parse.
\end{itemize}

\begin{align}
\label{ES0-1}%
&\tuple{\tuple{\V{S} ::= \mydot \V{RR}}, 0 }\text{@}0
    &&\phantom{\; \text{Accept}}
    &&\phantom{\quad \text{<(XX), (XX)>}}
    &\phantom{\; \text{!Used}} \\
\label{ES0-2}%
&\tuple{\tuple{\V{RR} ::= \mydot \V{x} \cat \V{RR}}, 0 }\text{@}0&&\\
\label{ES0-3}%
&\tuple{\tuple{\V{RR} ::= \mydot \V{x}}, 0 }\text{@}0&&
\end{align}
\medskip
\begin{align}
\label{ES1-1}%
&\tuple{\tuple{\V{RR} ::= \V{x} \mydot \V{RR}}, 0 }\text{@}1
    &&\phantom{\; \text{Accept}}
    &&\phantom{\quad \text{<(XX), (XX)>}}
    &\phantom{\; \text{!Used}} \\
\label{ES1-2}%
&\tuple{\tuple{\V{RR} ::= \mydot \V{x} \cat \V{RR}}, 1 }\text{@}1&&\\
\label{ES1-3}%
&\tuple{\tuple{\V{RR} ::= \mydot \V{x}}, 1 }\text{@}1&\\
\label{ES1-4}%
&\tuple{\tuple{\V{RR} ::= \V{x} \mydot}, 0 }\text{@}1
    &&
    &&\quad \tuple{\eqref{ES0-3}, \text{'x'}}
    &\; \text{!Used} \\
\label{ES1-5}%
&\tuple{\tuple{\V{S} ::= \V{RR} \mydot}, 0 }\text{@}1
    &&\; \text{Accept}
    &&\quad \tuple{\eqref{ES0-1}, \eqref{ES1-4}}
    &\; \text{!Used}
\end{align}
\medskip
\begin{align}
\label{ES2-1}%
&\tuple{\tuple{\V{RR} ::= \V{x} \mydot \V{RR}}, 1 }\text{@}2
    &&\phantom{\; \text{Accept}}
    &&\phantom{\quad \text{<(XX), (XX)>}}
    &\phantom{\; \text{!Used}} \\
\label{ES2-2}%
&\tuple{\tuple{\V{RR} ::= \mydot \V{x} \cat \V{RR}}, 2 }\text{@}2&\\
\label{ES2-3}%
&\tuple{\tuple{\V{RR} ::= \mydot \V{x}}, 2 }\text{@}2&\\
\label{ES2-4}%
&\tuple{\tuple{\V{RR} ::= \V{x} \mydot}, 1 }\text{@}2
    &&\; \text{Bottom}
    &&\quad \tuple{\eqref{ES1-3}, \text{'x'}}
    &\; \text{!Used} \\
\label{ES2-5}%
&\tuple{\tuple{\V{RR} ::= \V{x} \cat \V{RR} \mydot}, 0 }\text{@}2
    &&\; \text{Top}
    &&\quad \tuple{\eqref{ES1-1}, \eqref{ES2-4}}
    &\; \text{!Used} \\
\label{ES2-6}%
&\tuple{\tuple{\V{S} ::= \V{RR} \mydot}, 0 }\text{@}2
    &&\; \text{Accept}
    &&\quad \tuple{\eqref{ES0-1}, \eqref{ES2-5}}
    &\; \text{!Used}
\end{align}
\medskip
\begin{align}
\label{ES3-1}%
&\tuple{\tuple{\V{RR} ::= \V{x} \mydot \V{RR}}, 2 }\text{@}3
    &&\phantom{\; \text{Accept}}
    &&\phantom{\; \text{(XX), (XX)}}
    &\phantom{\; \text{!Used}} \\
\label{ES3-2}%
&\tuple{\tuple{\V{RR} ::= \mydot \V{x} \cat \V{RR}}, 3 }\text{@}3&\\
\label{ES3-3}%
&\tuple{\tuple{\V{RR} ::= \mydot \V{x}}, 3 }\text{@}3&\\
\label{ES3-4}%
&\tuple{\tuple{\V{RR} ::= \V{x} \mydot}, 2 }\text{@}3
    &&\; \text{Bottom}
    &&\quad \tuple{\eqref{ES2-3}, \text{'x'}}
    &\; \text{!Used} \\
\label{ES3-5}%
&\tuple{\tuple{\V{RR} ::= \V{x} \cat \V{RR} \mydot}, 1 }\text{@}3
    &&\; \text{Interior}
    &&\quad \tuple{\eqref{ES2-1}, \eqref{ES3-4}}
    &\; \text{!Used} \\
\label{ES3-6}%
&\tuple{\tuple{\V{RR} ::= \V{x} \cat \V{RR} \mydot}, 0 }\text{@}3
    &&\; \text{Top}
    &&\quad \tuple{\eqref{ES1-1}, \eqref{ES3-5}}
    &\; \text{!Used} \\
\label{ES3-7}%
&\tuple{\tuple{\V{S} ::= \V{RR} \mydot}, 0 }\text{@}3
    &&\; \text{Accept}
    &&\quad \tuple{\eqref{ES0-1}, \eqref{ES3-6}}
    &\; \text{!Used}
\end{align}
\medskip
\begin{align}
\label{ES4-1}%
&\tuple{\tuple{\V{RR} ::= \V{x} \mydot \V{RR}}, 3 }\text{@}4
    &&\phantom{\; \text{Accept}}
    &&\phantom{\; \text{(XX), (XX)}}
    &\phantom{\; \text{!Used}} \\
\label{ES4-2}%
&\tuple{\tuple{\V{RR} ::= \mydot \V{x} \cat \V{RR}}, 4 }\text{@}4\\
\label{ES4-3}%
&\tuple{\tuple{\V{RR} ::= \mydot \V{x}}, 4 }\text{@}4\\
\label{ES4-4}%
&\tuple{\tuple{\V{RR} ::= \V{x} \mydot}, 3 }\text{@}4
    &&\; \text{Bottom}
    &&\quad \tuple{\eqref{ES3-3}, \text{'x'}}
    &\; \text{Used}\\
\label{ES4-5}%
&\tuple{\tuple{\V{RR} ::= \V{x} \cat \V{RR} \mydot}, 2 }\text{@}4
    &&\; \text{Interior}
    &&\quad \tuple{\eqref{ES3-1}, \eqref{ES4-4}}
    &\; \text{Used}\\
\label{ES4-6}%
&\tuple{\tuple{\V{RR} ::= \V{x} \cat \V{RR} \mydot}, 1 }\text{@}4
    &&\; \text{Interior}
    &&\quad \tuple{\eqref{ES2-1}, \eqref{ES4-5}}
    &\; \text{Used}\\
\label{ES4-7}%
&\tuple{\tuple{\V{RR} ::= \V{x} \cat \V{RR} \mydot}, 0 }\text{@}4
    &&\; \text{Top}
    &&\quad \tuple{\eqref{ES1-1}, \eqref{ES4-6}}
    &\; \text{Used}\\
\label{ES4-8}%
&\tuple{\tuple{\V{S} ::= \V{RR} \mydot}, 0 }\text{@}4
    &&\; \text{Accept}
    &&\quad \tuple{\eqref{ES0-1}, \eqref{ES4-7}}
    &\; \text{Used}
\end{align}
\renewcommand{\firstRReqref}{\eqref{ES0-1}}
\renewcommand{\lastRReqref}{\eqref{ES4-8}}
\bigskip\par

A glance at \firstRReqref-\lastRReqref{} shows
that many of the completions are not used in the
parse.
This seems wasteful, and we wonder if this waste
can be avoided.

Five of the unused completions are accept \type{EIMT}'s,
that is, \type{EIMT}'s with completed start rules.
There is at most one of these per \type{ES},
so that the overhead is \On, and small.
A mechanism for eliminating useless accept rules
would gain us little,
and would have to, itself, come at a very small cost
to be justified.
We therefore look elsewhere.

For our analysis of the other useless completions,
we will need some new conceptual tools.
We introduce the concepts of Leo uniqueness,
Leo eligibility,
and Leo stack.

We say that an \type{EIMT} \Veimt{x} is \dfn{Leo unique}
in the Earley set \Ves{es},
$\LeoUnique{\var{x}, \var{es}}$,
iff
it is a penult in \var{es},
and its postdot symbol is unique in \var{es}.
More precisely,
\begin{equation}
\label{e:def-leo-unique}
\begin{gathered}
    \LeoUnique{\Veimt{eim}, \Ves{es}} \defined \\
        \V{eim} = \; \iotaQ{\Vtyped{eim2}{\type{EIMT}} }{
	    \V{eim2} \in \V{es} \land \Penult{\V{eim2}} \isWellDef }
\end{gathered}
\end{equation}
In \eqref{e:def-leo-unique}
it is important to emphasize that \Veimt{eim2} ranges over all the
\type{EIMT}'s, not just the penults.
This means that if a penult \Veimt{pen} shares a postdot symbol
with an non-penult \Veimt{np} in the same Earley set,
then \Veimt{pen} is {\bf not} Leo unique.

\begin{MYsloppyB}{3em}{2000}
If \Veimt{eim} in
\eqref{e:def-leo-unique}
is Leo unique in an \type{ES},
then the symbol $\Postdot{\var{eim}}$ is
also said to be \dfn{Leo unique} in that \type{ES}.
For each Leo unique symbol, call it \V{transition}, in an \type{ES},
call it \V{es},
there is exactly one dotted rule, call it \V{dr},
and exactly one rule, call it \V{r}.
We call \Vdr{dr}, the dotted rule for \Vsym{transition} in \Ves{es}.
We call \Vrule{r}, the rule for \Vsym{transition} in \Ves{es}.
\end{MYsloppyB}

An \type{EIMT} \Veimt{x}
is \dfn{Leo eligible}
in Earley set \Ves{es},
$$\LeoEligible{ \Veimt{x}, \Ves{es} },$$
iff it is right recursive and Leo unique
in \Ves{es}.
More precisely,
\begin{gather*}
\LeoEligible{ \Veimt{x}, \Ves{es} } \defined \\
\RightRecursive{\Rule{\Veimt{x}}}
\land \LeoUnique{\Veimt{x}, \Ves{es}}.
\end{gather*}
An \type{EIMT} is a
\dfn{Leo completion}
iff the mainstem of one of its confluences is
Leo eligible.

An \type{EIMT}, call it \V{eim1},
is a \dfn{Leo tributary} of another \type{EIMT}, call it \V{eim2},
in the Earley set at \Vloc{current}
iff \V{eim2} has a confluence whose tributary is \Veimt{eim1},
and whose mainstem is Leo eligible at the origin of \V{eim1}.
More precisely, \V{eim1} is a Leo tributary of \V{eim2} iff
\[
\begin{gathered}
\existQ{ \VTtyped{main}{EIMT} }{ \\
  \tuple{\V{main},\V{eim1}} \in \Confluences{\V{eim2},\V{current}} \\
  \land \; \LeoEligible{\V{main}, \Origin{\V{eim1}}}.
}
\end{gathered}
\]

A \dfn{Leo tributary sequence} is a tributary sequence in which
every term except the top is a tributary of its successor.
Every term of a Leo tributary sequence is in the same Earley set,
so that it is intuitive to
visualize a Leo tributary sequence as a vertical stack.

A \dfn{Leo stack} is a Leo tributary sequence such that
all of the following are true:
\begin{itemize}
\item The bottom of the sequence is not a Leo completion.
\item The top of the sequence is not a Leo tributary of any \type{EIMT}.
\item The top and bottom are distinct.
\end{itemize}

A Leo stack must have at least two terms,
and all but the bottom term will be a Leo completion.
Joop Leo's technique for eliminating useless items from an Earley table
was based on the following insight:
\begin{center}
Every Leo completion in a Leo stack,\linebreak
except the top of the stack,\linebreak
can be deduced from the top of the stack.\footnote{%
   For this reason,
   in his 1991~\cite{Leo1991},
   Leo's term for his concept analogous to our Leo stack was
   ``deterministic reduction path''.}
\end{center}
This meant that all items in a Leo stack, except the top and bottom
can be memoized during the parse,
and ignored if they turn to be useless.

Where \V{n} is the size of the largest Leo stack in
a parse using the grammar in \eqref{eq:leo-example-grammar},
the number of ``Leo memoizable'' \type{EIMT}'s is
\begin{equation}
\label{eq:memoizable-count}
\sum \seqB{2 \le \V{i} \le \V{n}}{ \V{i} - 2 }
\end{equation}
of which
\begin{equation}
\label{eq:ignorable-count}
\sum \seqB{2 \le \V{i} < \V{n}}{ \V{i} - 2 }
\end{equation}
can be ignored.

For the example parse of \firstRReqref-\lastRReqref{},
the number of items which can be memoized is 3,
but 2 of these need to be used for evaluation,
so that we actually save the processing of only one \type{EIMT}.
This is, frankly, unimpressive.

But \eqref{eq:memoizable-count} and
\eqref{eq:ignorable-count} are both \order{\V{n}^2},
and for larger \V{n} the effect of quadratic growth
takes over quickly.
Continuing to use
the grammar in \eqref{eq:leo-example-grammar},
and to let \V{n} be the size of the largest Leo stack,
when $\V{n} \ge 4$, the number of ignorable \type{EIMT}'s
is
\[
\genfrac{}{}{1pt}{}{ (\V{n} \subtract 3)\mult(\V{n} \subtract 2) }{ 2}
\]
and the number of other \type{EIMT}'s works out to
$7 \mult \V{n}.$
For $\V{n}=19$, over half of the \type{EIMT}'s can be ignored.

Right recursions are often very long,
so Leo memoization, if the cost is \Oc
with small constants,
is a clear win for many grammars in practical use.
Beginning in section \eqref{sec:leo-items},
we shall outline such a low-cost method.

\subsection{Differences with Leo 1991}

Our version of Leo memoization is somewhat
different from that in Leo 1991~\cite{Leo1991}.
  \begin{itemize}
  \item \Marpa's Leo memoization is eager,
  while that of~\cite{Leo1991} is lazy.
  \item \Marpa only does Leo memoization
  for right recursive rules.
  \end{itemize}
Omission of Leo memoization does not affect correctness,
so these changes
preserve the correctness as shown in \cite{Leo1991}.
And, later in this paper, we will
show that these changes also leave
the complexity results of
\cite{Leo1991} intact.

Looking more closely at our first difference,
the algorithm of \cite{Leo1991}
in some cases delayed Leo memoization
until after later \type{ES}'s were constructed.
It was not clear to us that this produced any savings.
It did make the algorithm more complex,
and presented a real obstacle to \Marpa's on-the-fly
features, such as event generation.
For these reasons, in \Marpa,
Leo memoization is eager.

Our second difference was to
consider Leo memoization only for
right recursive \type{EIMT}'s.
In \cite{Leo1991}, any penult
was subject to Leo memoization,
not just right recursions.
(We recall that
an \type{EIMT} \Vdr{eim} is a penult if $\Penult{\var{eim}} \isWellDef$.)

By restricting Leo memoization to right-recursive rules,
\Marpa{} incurs the cost of Leo memoization only in cases
where Leo sequences can be infinitely
long.
This more careful targeting of the memoization is for efficiency reasons.
If all penults are memoized,
memoizations will be performed where
the longest Leo stack is finite,
so that the payoff is limited.
and often quite small.
A possible future optimization would be to identify
non-right-recursive rules
which generate Leo stacks which
are long enough to
justify inclusion in the Leo memoizations.
But research would be needed to show that such an
optimization is worthwhile.

\subsection{Leo items}
\label{sec:leo-items}

A traditional Leo item (\type{LIMT}) is a triple
consisting of a dotted rule,
a symbol (called the transition symbol),
and a location.
\begin{gather*}
\DR{\Vlimt{limt}} \defined \Hd{\V{limt}}. \\
\Transition{\Vlimt{limt}} \defined \HTl{\V{limt}}. \\
\Origin{\Vlimt{limt}} \defined \TTl{\V{limt}}.
\end{gather*}
As will be explained in more detail later,
the \type{LIMT} \Vlimt{limt} indicates that
\begin{equation*}
\typed{\tuple{\DR{\Vlimt{limt}}, \Origin{\Vlimt{limt}}}}{\type{EIMT}}
\end{equation*}
is to be added on Leo reductions over
the symbol
$$\Transition{\Vlimt{limt}}.$$

Pedantically, \type{LIMT}'s are not members of Earley sets,
so we introduce a partial function from pairs of location
and symbol to \type{LIMT}'s:
$$
  \mymathop{LIMT-Map} : \naturals \times \var{vocab} \rightharpoonup \type{LIMT}.
$$
That $\mymathop{LIMT-Map}$ is a function implies that,
in each Earley set, there is at most one Leo item per symbol.

In practice we will usually avoid direct reference to $\mymathop{LIMT-Map}$,
finding it convenient and natural to overload
the notion of, and notation for, Earley set membership.
Where an Earley set \Ves{i}
and a \type{LIMT} \Vlimt{x} are such that
\begin{equation}
\label{e:limtmap-overload}
\exists \, \var{postdot} \in \var{vocab} \mid
     \var{x} = \LIMTMap{\Vloc{i}, \Vsym{postdot}},
\end{equation}
then we will often say that \Vlimt{x} is \dfn{in}
the Earley set \Ves{i},
and write
$$\Vlimt{x} \in \Ves{i}.$$
Accordingly, we have spoken of parse items, which may be Leo items,
being ``added to the Earley set'', and we will continue to speak in this way.
We note that in \eqref{e:limtmap-overload},
$\mymathop{LIMT-Map}$ is a partial function,
and therefore its value is not necessarily defined.

Implementing the Leo logic requires
adding Leo reduction as a new basic operation,
adding a new premise to the Earley reduction
operation,
and extending the Earley sets to memoize Earley
items as \type{LIMT}'s.

\subsection{\V{LIMT} mainstems}

In Section
\ref{s:reduction},
we defined the Earley mainstem of an \type{EIMT}.
We define the \type{LIMT} mainstem of an \type{EIMT} by analogy.
When say that \Vlimt{stem} is the \type{LIMT} mainstem of \Veimt{x}
if and only if
$\LIMTMainstem{\Vlimt{stem},\Veimt{x}}$ is true, where
\begin{gather*}
\LIMTMainstem{\Vlimt{stem},\Veimt{x}} \defined \\
\Transition{\Vlimt{stem}} = \LHS{\Veimt{x}} \\
\land \; \Vlimt{stem} \in \es{(\Origin{\Veimt{x}})}.
\end{gather*}

\subsection{Leo reduction}

\begin{equation*}
\inference{
    \begin{array}{c}
    \Veimt{tributary} \in \Ves{current} \\
    \LIMTMainstem{\Vlimt{mainstem},\var{tributary}} \\
    \end{array}
}{
    \begin{array}{c}
	\VTtyped{\tuple{ \DR{\Vlimt{mainstem}}, \Origin{\var{mainstem}} }}{EIMT} \\
	\VTtyped{\tuple{ \Vlimt{mainstem}, \Veimt{tributary}}}{CFLU}
    \end{array}
}
\end{equation*}
The new Leo reduction operation resembles the Earley reduction
operation, except that it looks for a mainstem \type{LIMT},
instead of an \type{EIMT}.
$\LHS{\Veimt{tributary}}$ is the transition symbol
of the Leo reduction.
A confluence which results from Leo reduction is
called a \dfn{Leo confluence}.
The mainstem of a Leo confluence is always a \type{LIMT}.

\subsection{Revised Earley reduction}

The Earley reduction of
\ref{s:reduction}
still applies, with an additional premise:

\begin{equation*}
\inference{
    \begin{array}{c}
    \Veimt{tributary} \in \Ves{current} \\
    \nexistQ{ \VTtyped{x}{LIMT} }{ \LIMTMainstem{\var{x},\var{tributary}}} \\
    \Veimt{mainstem} \in \es{(\Origin{\Veimt{tributary}})} \\
    \Postdot{\Veimt{mainstem}} = \LHS{\Veimt{tributary}} \\
    \end{array}
}{
    \begin{array}{c}
	\VTtyped{ \tuple{ \Next{\Veimt{mainstem}}, \Origin{\var{mainstem}} }}{EIMT} \\
	\VTtyped{\tuple{\Veimt{mainstem}, \Veimt{tributary} }}{CFLU}.
    \end{array}
}
\end{equation*}

The additional premise
prevents Earley reduction from being applied
where there is an \type{LIMT} with
$\LHS{\Veimt{tributary}}$
as its transition symbol.
This reflects the fact that
Leo reduction replaces Earley reduction if and only if
there is a Leo memoization.

\subsection{Leo memoization}

We are now ready to define the inference rules
for Leo memoization.
We define one rule
that holds if
a \type{LIMT} mainstem can be found for the tributary \type{EIMT}
in \Ves{current},
\begin{equation*}
\inference{
    \begin{array}{c}
    \LIMTMainstem{\Vlimt{stem},\Veimt{trib}} \\
    \LeoEligible{\Veimt{trib}, \Ves{current}}
    \end{array}
}{
    \begin{array}{c}
	\VTtyped{ \tuple{\DR{\Vlimt{stem}}, \Penult{\Veimt{trib}}, \Origin{\var{stem}}}}{LIMT} \\
	\VTtyped{\tuple{\Vlimt{stem},\Veimt{trib}}}{CFLU}
    \end{array}
}
\end{equation*}
and another inference rule that holds if
\Veimt{trib} has no mainstem \type{LIMT},
\begin{equation*}
\inference{
    \begin{array}{c}
    \univQ{ \VTtyped{stem}{LIMT} }{ \neg \LIMTMainstem{\V{stem},\Veimt{trib}}} \\
    \LeoEligible{\Veimt{trib}, \Ves{current}}
    \end{array}
}{
    \begin{array}{c}
	\VTtyped{ \tuple{\DR{\Veimt{trib}}, \Penult{\V{trib}}, \Origin{\V{trib}}}}{LIMT} \\
	\VTtyped{\tuple{\unicorn,\Veimt{trib}}}{CFLU}.
    \end{array}
}
\end{equation*}
We note that when the confluence of a \type{LIMT} has no mainstem,
its mainstem is ill-defined.

\begin{algorithm}[ht]
\caption{Leo stack expansion}
\label{alg:leo-expansion}
\begin{algorithmic}[1]
\Procedure{LeoExpand}{$\VTtyped{top}{EIMT},\VTtyped{confl}{CFLU},\VTtyped{current}{ES}$}
\Comment{Expand Leo stack for $\V{confl}\in\Confluences{\V{top},\V{current}}$.
    \V{top} may have more than one Leo confluence and this algorithm must be run for each of them.}
\State $\tuple{\Vlimt{trunkLIM},\Veimt{previousCompletion}} \gets \V{confl}$
\While{$\Vlimt{trunkLIM}\isWellDef$}
\State $\VTtyped{\tuple{\Vlimt{nextTrunkLIM}, \Veimt{penult}}}{CFLU}$
\Statex \hspace\algorithmicindent\hspace{2em}
  $\gets\; \iotaQ{\V{c}}{\V{c}\in\Confluences{\V{trunkLIM},\V{current}}}$
\Statex \Comment{\type{LIMT}'s always have exactly one confluence.}
\State $\Veimt{completion} \gets \tuple{\Next{\DR{\Veimt{penult}}}, \Origin{\V{penult}}}$
\State Add completion \Veimt{completion} to \Ves{current}.
\State Add $\VTtyped{\tuple{\Veimt{penult}, \V{previousCompletion}}}{CFLU}$
\Statex \hspace\algorithmicindent\hspace{2em} to $\Confluences{\V{completion}, \Ves{current}}$
\Statex \Comment{"Adds" are to sets --- elements already in the set must not be duplicated.}
\State $\Vlimt{trunkLIM} \gets \Vlimt{nextTrunkLIM}$
\State $\Veimt{previousCompletion} \gets \Veimt{completion}$
\EndWhile
\EndProcedure
\end{algorithmic}
\end{algorithm}

\subsection{Leo evaluation}

\Marpa evaluation deals with Leo-memoized \type{EIMT}'s by expanding
those Leo stacks which contain memoized \type{EIMT}'s that are useful.
For each right recursion, there will be only one such Leo stack.
In the example parse of \firstRReqref-\lastRReqref{},
that Leo stack is in \type{ES} 4.

Algorithm \ref{alg:leo-expansion}
on page~\pageref{alg:leo-expansion}
recreates the Leo stack for a
confluence of a top \type{EIMT}.
We note that top of a Leo stack may have more than
one confluence, and that the others may be either Leo or ordinary confluences.
Algorithm \ref{alg:leo-expansion} must be run for each of the Leo confluences.

\section{The Aycock-Horspool finite automaton}
\label{s:AHFA}
\label{s:end-prelim}

In this paper a
``split LR(0) $\epsilon$-DFA''
as described by Aycock and Horspool~\cite{AH2002},
will be called an Aycock-Horspool Finite Automaton,
or AHFA.
This section will
summarize the ideas
from~\cite{AH2002}
that are central to \Marpa.

Aycock and Horspool based their AHFA's
on a few observations.
\begin{itemize}
\item
In practice, Earley items sharing the same origin,
but having different dotted rules,
often appear together in the same Earley set.
\item
There is in the literature a method
for associating groups of dotted rules that often appear together
when parsing.
This method is the LR(0) DFA used in the much-studied
LALR and LR parsers.
\item
The LR(0) items that are the components of LR(0)
states are, exactly, dotted rules.
\item
By taking into account symbols that derive the
null string, the LR(0) DFA could be turned into an
LR(0) $\epsilon$-DFA,
which would be even more effective
at grouping dotted rules that often occur together
into a single DFA state.
\end{itemize}

AHFA states are, in effect,
a shorthand
for groups of dotted rules that occur together frequently.
Aycock and Horspool realized that,
by changing Earley items to track AHFA states
instead of individual dotted rules,
the size of Earley sets could be reduced,
and conjectured that this would
make Earley's algorithm faster in practice.

\begin{MYsloppy}
As a reminder,
the original Earley items (\type{EIMT}'s)
were duples, $\tuple{\Vdr{x}, \Vorig{x}}$,
where \Vdr{x} is a dotted rule.
An Aycock-Horspool Earley item is a duple
\begin{equation*}
\tuple{\Vah{y}, \Vorig{y}},
\end{equation*}
where $\Vah{y}$ is an AHFA state.
\end{MYsloppy}

\Marpa uses
Earley items of the form
created by Aycock and Horspool.
A \Marpa Earley item has type \type{EIM},
and a \Marpa Earley item is often referred to as an \type{EIM}.

\begin{MYsloppy}
Aycock and Horspool did not consider
Leo's modifications,
but \Marpa incorporates them,
and \Marpa also changes its Leo items to use AHFA states.
Marpa's Leo items (\type{LIM}'s) are triples
of the form
\begin{equation*}
\tuple{\Vah{top}, \Vsym{transition}, \Vorig{top}},
\end{equation*}
where \Vsym{transition} and \Vorig{top}
are as in the traditional Leo items,
and \Vah{top} is an AHFA state.
A \Marpa Leo item has type \type{LIM}.
\end{MYsloppy}

\cite{AH2002} also defines
a partial transition function for
pairs of AHFA state and symbol,
\begin{equation*}
\GOTO: \Cfa, (\epsilon \cup \var{vocab}) \mapsto \Cfa.
\end{equation*}
$\GOTO(\Vah{from}, \epsilon)$ is a
\dfn{null transition}.
(AHFA's are not fully deterministic.)
If \Vah{predicted} is the result of a null transition,
it is called a \dfn{predicted} AHFA state.
If an AHFA state is not a \dfn{predicted} AHFA state,
it is called a \dfn{confirmed} AHFA state.
The initial AHFA state is a confirmed AHFA state.\footnote{%
In~\cite{AH2002} confirmed states are called ``kernel states'',
and predicted states are called ``non-kernel states''.
}

The states of an AHFA
are not a partition of the dotted
rules --
a single dotted rule can occur
in more than one AHFA state.
In combining
the improvements of Leo~\cite{Leo1991} and
Aycock and Horspool~\cite{AH2002},
the following theorem is crucial.

\begin{theorem}\label{t:leo-singleton}
If a \textnormal{\Marpa} Earley item (\type{EIM}) is the result of a
Leo reduction,
then its AHFA state contains only one dotted rule.
\end{theorem}

\begin{proof}
Let the \type{EIM} that is the result of the Leo
reduction be
\begin{equation*}
\Veim{result} = \tuple{\Vah{result}, \Vorig{result}}.
\end{equation*}
Let the Earley set that contains \Veim{result} be
\Ves{i}.
Since \Veim{result} is the result of a Leo reduction
we know, from the definition of a Leo reduction,
that
\begin{equation*}
\Vdr{complete} \in \Vah{result}
\end{equation*}
where
\Vdr{complete} is a completed rule.
Let
\begin{equation*}
\begin{gathered}
\Vrule{c} = \Rule{\Vdr{complete}} \text{and } \V{cp} = \Pos{\V{complete}}, \\
\text{so that } \Vdr{complete} = \tuple{ \Vrule{c}, \var{cp} }.
\end{gathered}
\end{equation*}
We note that
$\var{cp} > 0$ because, in \Marpa{}
grammars, completions are never
predictions.

Suppose, for a reduction to absurdity,
that the AHFA state contains another dotted rule,
\Vdr{other}, that is, that
\begin{equation*}
\Vdr{other} \in \Vah{result},
\end{equation*}
where $\Vdr{complete} \neq \Vdr{other}$.
Let \Vrule{o} be the rule of \Vdr{other},
and \var{op} its dot position,
\begin{equation*}
\Vdr{other} = \tuple{ \Vrule{o}, \var{op} }.
\end{equation*}
AHFA construction never places a prediction in the same
AHFA state as a completion, so
\Vdr{other} is not a prediction.
Therefore, $\var{op} > 0$.
To show this outer reduction to absurdity, we first prove
by a first inner reductio that
$\Vrule{c} \neq \Vrule{o}$,
then by a second inner reductio that
$\Vrule{c} = \Vrule{o}$.

Assume, for the first inner reductio,
that
$\Vrule{c} = \Vrule{o}$.
By the construction of an AHFA
state,
both \Vdr{complete} and \Vdr{other}
resulted from the same series
of transitions.
But the same series of transitions over the
same rule would result in the same dot position,
$\var{cp} = \var{op}$,
so that if $\Vrule{c} = \Vrule{o}$,
$\Vdr{complete} = \Vdr{other}$,
which is contrary to the assumption for the outer reductio.
This shows the first inner reductio.

Next, we assume for the second inner reductio that
$\Vrule{c} \ne \Vrule{o}$.
Since both \Vdr{complete} and \Vdr{other}
are in the same \type{EIM}
and neither is a prediction,
both must result from transitions,
and their transitions must have been from the same Earley set.
Since they are in the same AHFA state,
by the AHFA construction,
that transition must have been
over the same transition symbol,
call it \Vsym{transition}.
But Leo uniqueness applies to \Vdr{complete},
and requires that the transition
over \Vsym{transition} be unique in \Ves{i}.

But if $\Vrule{c} \ne \Vrule{o}$,
\Vsym{transition} was the transition symbol
of two different dotted rules,
and the Leo uniqueness requirement does not hold.
The conclusion that the Leo uniqueness requirement
both does and does not hold
is a contradiction,
which shows the second inner reductio.
Since the assumption for
the second inner reductio was that
$\Vrule{c} \ne \Vrule{o}$,
we conclude that
$\Vrule{c} = \Vrule{o}$.

By the two inner reductio's,
we have both
$$\Vrule{c} \neq \Vrule{o} \text{ and } \Vrule{c} = \Vrule{o},$$
which completes the outer reduction to absurdity.
For the outer reductio, we assumed that
\Vdr{other}
was a second dotted rule in \Vah{result},
such that
$\Vdr{other} \neq \Vdr{complete}$.
We can therefore conclude that
\begin{equation*}
\Vdr{other} \in \Vah{result} \implies \Vdr{other} = \Vdr{complete}.
\end{equation*}
If \Vdr{complete} is a dotted rule
in the AHFA state of a Leo reduction \type{EIM},
then it must be the only dotted rule in that AHFA state.
\end{proof}

\section{The \Marpa recognizer}
\label{s:recce}
\label{s:pseudocode}

\subsection{Complexity}

Alongside the pseudocode of this section
are observations about its space and time complexity.
In what follows,
we will charge all time and space resources
to Earley items,
or to attempts to add Earley items.
We will show that,
to each Earley item actually added,
or to each attempt to add a duplicate Earley item,
we can charge amortized \Oc{} time and space.

At points, it will not be immediately
convenient to speak of
charging a resource
to an Earley item
or to an attempt to add a duplicate
Earley item.
In those circumstances,
we speak of charging time and space
\begin{itemize}
\item to the parse; or
\item to the Earley set; or
\item to the current procedure's caller.
\end{itemize}

We can charge time and space to the parse itself,
as long as the total time and space charged is \Oc.
Afterwards, this resource can be re-charged to
the initial Earley item, which is present in all parses.
Soft and hard failures of the recognizer use
worst-case \Oc{} resource,
and are charged to the parse.

We can charge resources to the Earley set,
as long as the time or space is \Oc.
Afterwards,
the resource charged to the Earley set can be
re-charged to an arbitrary member of the Earley set,
for example, the first.
If an Earley set is empty,
the parse must fail,
and the resource can be charged to the parse.

In a procedure,
resource can be ``caller-included''.
Caller-included resource is not accounted for in
the current procedure,
but passed upward to the procedure's caller,
to be accounted for there.
A procedure to which caller-included resource is passed will
sometimes pass the resource upward to its own caller,
although of course the top-level procedure does not do this.

For each procedure, we will state whether
the time and space we are charging is inclusive or exclusive.
The exclusive time or space of a procedure is that
which it uses directly,
ignoring resource charges passed up from called procedures.
Inclusive time or space includes
resource passed upward to the
current procedure from called procedures.

Recall that
Earley sets may be represented by \Ves{i},
where \var{i} is the Earley set's location \Vloc{i}.
The two notations should be regarded as interchangeable.
The actual implementation of either
should be the equivalent of a pointer to
a data structure containing,
at a minimum,
the Earley items,
a memoization of the Earley set's location as an integer,
and a per-set-list.
Per-set-lists will be described in Section \ref{s:per-set-lists}.

\begin{algorithm}[ht]
\caption{Marpa Top-level}
\label{alg:top-level}
\begin{algorithmic}[1]
\Procedure{Main}{}
\State \Call{Initial}{}
\For{ $\var{i}, 0 \le \var{i} \le \Vsize{w}$ }
\Statex \Comment At this point, $\Ves{x}$ is complete, for $0 \le \var{x} < \var{i}$
\State \Call{Scan pass}{$\var{i}, \var{w}[\var{i} \subtract 1]$}
\If{$\size{\Ves{i}} = 0$}
\State reject \Cw{} and return
\EndIf
\State \Call{Reduction pass}{\var{i}}
\EndFor
\For{every $\VTtyped{\tuple{\Vah{x}, 0}}{EIM} \in \Etable{\Vsize{w}}$}
\If{$\Vdr{accept} \in \Vah{x}$}
\State accept \Cw{} and return
\EndIf
\EndFor
\State reject \Cw{}
\EndProcedure
\end{algorithmic}
\end{algorithm}

\subsection{Top-level code}

The top-level code is
Algorithm~\ref{alg:top-level}
on page~\pageref{alg:top-level}.
Exclusive time and space for the loop over the Earley sets
is charged to the Earley sets.
Inclusive time and space for the final loop to
check for \Vdr{accept} is charged to
the Earley items at location \size{\Cw}.
Overhead is charged to the parse.
All these resource charges are obviously \Oc.

\subsection{Ruby Slippers parsing}
\begin{MYsloppy}
The top-level code of
Algorithm~\ref{alg:top-level}
(p.~\pageref{alg:top-level})
represents a significant change
from \AH~\cite{AH2002}.
\call{Scan pass}{} and \call{Reduction pass}{}
are separated.
As a result,
when the scanning of tokens that start at location \Vloc{i} begins,
the Earley sets for all locations prior to \Vloc{i} are complete.
This means that the scanning operation has available, in
the Earley sets,
full information about the current state of the parse,
including which tokens are acceptable during the scanning phase.
\end{MYsloppy}

\begin{algorithm}[ht]
\caption{Initialization}
\label{alg:initialization}
\begin{algorithmic}[1]
\Procedure{Initial}{}
\State \Call{Add \type{EIM} pair}{$\es{(0)}, \Vah{start}, 0$}
\EndProcedure
\end{algorithmic}
\end{algorithm}

\subsection{Initialization}
\label{p:initial-op}

The initialization code is
Algorithm~\ref{alg:initialization}
on page~\pageref{alg:initialization}.
Inclusive time and space is \Oc{}
and is charged to the parse.

\begin{algorithm}[ht]
\caption{Marpa Scan pass}
\label{alg:scan-pass}
\begin{algorithmic}[1]
\Procedure{Scan pass}{$\Vloc{i},\Vsym{a}$}
\For{each $\Veim{mainstem} \in \var{transitions}((\var{i} \subtract 1),\var{a})$}
\Statex \Comment{Each pass through this loop is an \type{EIM} attempt.}
\State $\tuple{\Vah{from}, \Vloc{origin}} \gets \Veim{mainstem}$
\State $\Vah{to} \gets \GOTO(\Vah{from}, \Vsym{a})$
\State \Call{Add \type{EIM} pair}{$\Ves{i}, \Vah{to}, \Vloc{origin}$}
\EndFor
\EndProcedure
\end{algorithmic}
\end{algorithm}

\subsection{Scan pass}
\label{p:scan-op}

The code for the scan pass is
Algorithm~\ref{alg:scan-pass}
on page~\pageref{alg:scan-pass}.
\var{transitions} is a set of tables, one per Earley set.
The tables in the set are indexed by symbol.
Symbol indexing is \Oc, since the number of symbols
is a constant, but
since the number of Earley sets grows with
the length of the parse,
it cannot be assumed that Earley sets can be indexed by location
in \Oc{} time.
For the operation $\var{transitions}(\Vloc{l}, \Vsym{s})$
to be in \Oc{} time,
\Vloc{l} must represent a link directly to the Earley set.
In the case of scanning,
the lookup is always in the previous Earley set,
which can easily be tracked in \Oc{} space
and retrieved in \Oc{} time.
Inclusive time and space can be charged to the
\Veim{mainstem}.
Overhead is charged to the Earley set at \Vloc{i}.

\begin{algorithm}[ht]
\caption{Reduction pass}
\label{alg:reduction-pass}
\begin{algorithmic}[1]
\Procedure{Reduction pass}{\Vloc{i}}
\Statex \Comment{\Vtable{i} may include \type{EIM}'s added by
    by \Call{Reduce one LHS}{} and the loop must traverse these.}
\For{each Earley item $\Veim{work} \in \Vtable{i}$}
\State $\tuple{\Vah{work}, \Vloc{origin}} \gets \Veim{work}$
\State $\Vsymset{lh-sides} \gets$ a set containing the LHS
\State \hspace\algorithmicindent of every completed rule in \Vah{work}
\For{each $\Vsym{lhs} \in \Vsymset{lh-sides}$}
\State \Call{Reduce one LHS}{\Vloc{i}, \Vloc{origin}, \Vsym{lhs}}
\EndFor
\EndFor
\State \Call{Memoize transitions}{\Vloc{i}}
\EndProcedure
\end{algorithmic}
\end{algorithm}

\subsection{Reduction pass}

The code for the reduction pass is
Algorithm~\ref{alg:reduction-pass}
on page~\pageref{alg:reduction-pass}.
The loop over \Vtable{i} must also include
any items added by \call{Reduce one LHS}{}.
This can be done by implementing \Vtable{i} as a list
and adding new items at the end.

Exclusive time is clearly \Oc{} per
\Veim{work},
and is charged to the \Veim{work}.
Additionally,
some of the time required by
\call{Reduce one LHS}{} is caller-included,
and therefore charged to this procedure.
Inclusive time from \call{Reduce one LHS}{}
is \Oc{} per call,
as will be seen in Section~\ref{p:reduce-one-lhs},
and is charged to the \Veim{work}
that is current
during that call to \call{Reduce one LHS}{}.
Overhead may be charged to the Earley set at \Vloc{i}.

\begin{algorithm}[ht]
\caption{Memoize transitions}
\label{alg:memoize-transitions}
\begin{algorithmic}[1]
\Procedure{Memoize transitions}{\Vloc{i}}
\For{every \Vsym{postdot}, a postdot symbol of $\Ves{i}$}
\If{$\LeoEligible{\Vsym{postdot},\Ves{i}}$}
\State Set $\var{transitions}(\Vloc{i},\Vsym{postdot})$ to a \type{LIM}
\Else
\State Set $\var{transitions}(\Vloc{i},\Vsym{postdot})$
\Statex \hspace{\algorithmicindent}\hspace{4em}
    to the set of \VTtyped{eim}{EIM} such that
\Statex \hspace{\algorithmicindent}\hspace{4em}
    $\V{eim} \in \Ves{i} \land \Postdot{\V{eim}} = \Vsym{postdot}$
\EndIf
\EndFor
\EndProcedure
\end{algorithmic}
\end{algorithm}

\subsection{Memoize transitions}

The code for the memoization of transitions is
Algorithm~\ref{alg:memoize-transitions}
on page~\pageref{alg:memoize-transitions}.
The \var{transitions} table for \Ves{i}
is built once all \type{EIM}s have been
added to \Ves{i}.
We first look at the resource,
excluding the processing of Leo items.
The non-Leo processing can be done in
a single pass over \Ves{i},
in \Oc{} time per \type{EIM}.
Inclusive time and space are charged to the
Earley items being examined.
Overhead is charged to \Ves{i}.

We now look at the resource used in the Leo processing.
A transition symbol \Vsym{transition}
is Leo eligible if it is Leo unique
and its rule is right recursive.
(If \Vsym{transition} is Leo unique in \Ves{i}, it will be the
postdot symbol of only one rule in \Ves{i}.)
All but one of the determinations needed to decide
if \Vsym{transition} is Leo eligible can be precomputed
from the grammar,
and the resource to do this is charged to the parse.
The precomputation, for example,
for every rule, determines if it is right recursive.

One part of the test for
Leo eligibility cannot be done as a precomputation.
This is the determination whether there is only one \type{EIM}
in \Ves{i} whose postdot symbol is
\Vsym{transition}.
This can be done
in a single pass over the \type{EIM}'s of \Ves{i}
that notes the postdot symbols as they are encountered
and whether any is encountered twice.
The time and space,
including that for the creation of a \type{LIM} if necessary,
will be \Oc{} time per \type{EIM} examined,
and can be charged to \type{EIM} being examined.

\begin{algorithm}[ht]
\caption{Reduce one LHS symbol}
\label{alg:reduce-one-LHS-symbol}
\begin{algorithmic}[1]
\Procedure{Reduce one LHS}{\Vloc{i}, \Vloc{origin}, \Vsym{lhs}}
\Statex \Comment{Each pass through the following loop is an \type{EIM} attempt}
\For{each $\var{pim} \in \var{transitions}(\Vloc{origin},\Vsym{lhs})$}
\Statex \Comment \var{pim} is a ``postdot item'', either a \type{LIM} or an \type{EIM}
\If{\var{pim} is a \type{LIM}, \Vlim{pim}}
\State \Call{Leo reduction}{\Vloc{i}, \Vlim{pim}}
\Else
\State \Call{Earley reduction}{\Vloc{i}, \Veim{pim}, \Vsym{lhs}}
\EndIf
\EndFor
\EndProcedure
\end{algorithmic}
\end{algorithm}

\subsection{Reduce one LHS}
\label{p:reduce-one-lhs}

The code to reduce a single LHS symbol is
Algorithm~\ref{alg:reduce-one-LHS-symbol}
on page~\pageref{alg:reduce-one-LHS-symbol}.
To show that
\begin{equation*}
\var{transitions}(\Vloc{origin},\Vsym{lhs})
\end{equation*}
can be traversed in \Oc{} time,
we note
that the number of symbols is a constant
and assume that \Vloc{origin} is implemented
as a link back to the Earley set,
rather than as an integer index.
This requires that \Veim{work}
in \call{Reduction pass}{}
carry a link
back to its origin.
As implemented in~\cite{Marpa-R2}, \Marpa's
Earley items have such links.

Inclusive time
for the loop over the \type{EIM} attempts
is charged to each \type{EIM} attempt.
Overhead is \Oc{} and caller-included.

\begin{algorithm}[ht]
\caption{Earley reduction}
\label{alg:earley-reduction}
\begin{algorithmic}[1]
\Procedure{Earley reduction}{\Vloc{i}, \Veim{from}, \Vsym{trans}}
\State $\tuple{\Vah{from}, \Vloc{origin}} \gets \Veim{from}$
\State $\Vah{to} \gets \GOTO(\Vah{from}, \Vsym{trans})$
\State \Call{Add \type{EIM} pair}{\Ves{i}, \Vah{to}, \Vloc{origin}}
\EndProcedure
\end{algorithmic}
\end{algorithm}

\subsection{Earley reduction operation}
\label{p:reduction-op}

The code that performs Earley reduction is
Algorithm~\ref{alg:earley-reduction}
on page~\pageref{alg:earley-reduction}.
Exclusive time and space is clearly \Oc.
\call{Earley reduction}{} is always
called as part of an \type{EIM} attempt,
and inclusive time and space is charged to the \type{EIM}
attempt.

\begin{algorithm}[ht]
\caption{Leo reduction}
\label{alg:leo-reduction}
\begin{algorithmic}[1]
\Procedure{Leo reduction}{\Vloc{i}, \Vlim{from}}
\State $\tuple{\Vah{from}, \Vsym{trans}, \Vloc{origin}} \gets \Vlim{from}$
\State $\Vah{to} \gets \GOTO(\Vah{from}, \Vsym{trans})$
\State \Call{Add \type{EIM} pair}{\Ves{i}, \Vah{to}, \Vloc{origin}}
\EndProcedure
\end{algorithmic}
\end{algorithm}

\subsection{Leo reduction operation}
\label{p:leo-op}

The code that performs Leo reduction is
Algorithm~\ref{alg:leo-reduction}
on page~\pageref{alg:leo-reduction}.
Exclusive time and space is clearly \Oc.
\call{Leo reduction}{} is always
called as part of an \type{EIM} attempt,
and inclusive time and space is charged to the \type{EIM}
attempt.

\begin{algorithm}[ht]
\caption{Add \type{EIM} pair}
\label{alg:add-EIM-pair}
\begin{algorithmic}[1]
\Procedure{Add \type{EIM} pair}{$\Ves{i},\Vah{confirmed}, \Vloc{origin}$}
\State $\Veim{confirmed} \gets \tuple{\Vah{confirmed}, \Vloc{origin}}$
\State $\Vah{predicted} \gets \GOTO(\Vah{confirmed}, \epsilon)$
\If{\Veim{confirmed} is new}
\State Add \Veim{confirmed} to \Vtable{i}
\EndIf
\If{$\Vah{predicted} \neq \Lambda$}
\State $\Veim{predicted} \gets \tuple{\Vah{predicted}, \Vloc{i}}$
\If{\Veim{predicted} is new}
\State Add \Veim{predicted} to \Vtable{i}
\EndIf
\EndIf
\EndProcedure
\end{algorithmic}
\end{algorithm}

\subsection{Adding a pair of Earley items}
\label{p:add-eim-pair}

The code in
Algorithm~\ref{alg:add-EIM-pair}
on page~\pageref{alg:add-EIM-pair}
attempts to add a pair of Earley items (EIMs),
one confirmed and the other a prediction.
Algorithm~\ref{alg:add-EIM-pair}
first attempts to add a confirmed \type{EIM}.
Then
Algorithm~\ref{alg:add-EIM-pair}
checks for the existence of \Veim{predicted}, the \type{EIM} for
the null-transition of the confirmed \type{EIM}.
If \Veim{predicted} exists,
Algorithm~\ref{alg:add-EIM-pair}
attempts to add \Veim{predicted}.

Inclusive time and space is charged to the
calling procedure.
Trivially, the space is \Oc{} per call.

We show that time is also \Oc{}
by singling out the two non-trivial cases:
checking that an Earley item is new,
and adding it to the Earley set.
\Marpa{} checks whether an Earley item is new
in \Oc{} time
by using a data structure called a PSL.
PSL's are the subject of Section \ref{s:per-set-lists}.
An Earley item can be added to the current
set in \Oc{} time
if Earley set is seen as a linked
list, to the head of which the new Earley item is added.

The resource used by \call{Add \type{EIM} Pair}{}
is always caller-included.
No time or space is ever charged
to a predicted Earley item.
At most one attempt to add a \Veim{predicted} will
be made per attempt to add a \Veim{confirmed},
so that the total resource charged
remains \Oc.

\subsection{Per-set lists}
\label{s:per-set-lists}

In the general case,
where \var{x} is an arbitrary datum,
it is not possible
to use duple $\tuple{\Ves{i}, \V{x}}$
as a search key and expect the search to use
\Oc{} time.
Within \Marpa, however, there are specific cases
where it is desirable to do exactly that.
This is accomplished by
taking advantage of special properties of the search.

If it can be arranged that there is
a link direct to the Earley set \Ves{i},
and that $0 \leq \var{x} < \var{c}$,
where \var{c} is a constant of reasonable size,
then a search can be made in \Oc{} time,
using a data structure called a PSL.
Data structures identical to or very similar to PSL's are
briefly outlined in both
\cite[p. 97]{Earley1970}~and~%
\cite[Vol. 1, pages 326-327]{AU1972}.
But neither source gives them a name.
The term PSL
(``per-Earley set list'')
is new
with this paper.

A PSL is a fixed-length array of
integers, indexed by an integer,
and kept as part of each Earley set.
While \Marpa{} is building a new Earley set,
\Ves{j},
the PSL for every previous Earley set, \Vloc{i},
tracks the Earley items in \Ves{j} that have \Vloc{i}
as their origin.
The maximum number of Earley items that must be tracked
in each PSL is
the number of AHFA states,
\Vsize{\Cfa},
which is a constant of reasonable size
that depends on \Cg{}.

It would take more than \Oc{} time
to clear and rebuild the PSL's each time
that a new Earley set is started.
This overhead is avoided by ``time-stamping'' each PSL
entry with the Earley set
that was current when that PSL
entry was last updated.

As before,
where \Ves{i} is an Earley set,
let \Vloc{i} be its location,
and vice versa.
\Vloc{i} is an integer which is
assigned as Earley sets are created.
Let $\ID{\Vah{x}}$ be the integer ID of an AHFA state.
Numbering the AHFA states from 0 on up as they are created
is an easy way to create $\ID{\Vah{x}}$.
Let $\PSL{\Ves{x}}{\var{y}}$
be the entry for integer \var{y} in the PSL in
the Earley set at \Vloc{x}.

Consider the case where \Marpa is building \Ves{j}
and wants to check whether Earley item
$\Veim{x} = \tuple{\Vah{x}, \Vorig{x}}$ is new.
\Marpa{} looks at
\begin{equation*}
\Vloc{time-stamp} = \PSL{\Ves{x}}{\ID{\Vah{x}}},
\end{equation*}
and proceeds as follows:

\begin{itemize}
\item PSL entries are initially undefined.
If \Vloc{time-stamp} is undefined,
then the entry has never been used,
and \Veim{x} is new.
\Veim{x} will be added to \Ves{j}
and the time stamp will be reset.
\item If $\var{time-stamp} = \Vloc{j},$
then \Veim{x} is not new,
and will not be added to \Ves{j}.
The time stamp is left as it is.
\item If
$\var{time-stamp} \ne \Vloc{j}$,
then \Veim{x} is new.
\Veim{x} will be added to \Ves{j}
and the time stamp will be reset.
\end{itemize}

Resetting the time stamp is done as follows:
$$\PSL{\Ves{x}}{\ID{\Vah{x}}} \gets \Vloc{j}.$$
\subsection{Complexity summary}

For convenience, we collect and summarize here
some of the observations of this section.

\begin{observation}
The time and space charged to an Earley item
which is actually added to the Earley sets
is \Oc.
\end{observation}

\begin{observation}
The time charged to an attempt
to add a duplicate Earley item to the Earley sets
is \Oc.
\end{observation}

For evaluation purposes, \Marpa{} adds a confluence to
each \type{EIM} for every attempt to
add that \type{EIM},
even if that \type{EIM}
is a duplicate.
Traditionally, complexity results treat parsers
as recognizers, and such costs are ignored.
This will be an issue when the space complexity
for unambiguous grammars is considered.

\begin{observation}
The space charged to an attempt
to add a duplicate Earley item to the Earley sets
is \Oc{} if the confluences are included,
zero otherwise.
\end{observation}

As noted in Section \ref{p:add-eim-pair},
the time and space used by predicted Earley items
and attempts to add them is charged elsewhere.

\begin{observation}
No space or time is charged to predicted Earley items,
or to attempts to add predicted Earley items.
\end{observation}

\section{Preliminaries to the theoretical results}
\label{s:proof-preliminaries}

\subsection{Nulling symbols}
\label{s:nulling}

Recall that \Marpa grammars,
without loss of generality,
contain neither empty rules or
properly nullable symbols.
This corresponds directly
to a grammar rewrite in the implementation of~\cite{Marpa-R2},
and its reversal during \Marpa's evaluation phase.
For the correctness and complexity proofs in this paper,
we assume an additional rewrite,
this time to eliminate nulling symbols.

Elimination of nulling symbols is also
without loss of generality, as can be seen
if we assume that a history
of the rewrite is kept,
and that the rewrite is reversed
after the parse.
Clearly, whether a grammar \Cg{} accepts
an input \Cw{}
will not depend on the nulling symbols in its rules.

In~\cite{Marpa-R2},
\Marpa{} does not directly rewrite the grammar
to eliminate nulling symbols.
But nulling symbols are ignored in
creating the AHFA states,
and must be restored during \Marpa's evaluation phase,
so that the implementation of~\cite{Marpa-R2} and
this simplification for theory purposes
track each other closely.

\subsection{Comparing Earley items}

\begin{definition}
A \textnormal{\Marpa} Earley item \dfn{corresponds}
to a traditional Earley item
$\Veimt{x} = \tuple{\Vdr{x}, \Vorig{x}}$
if and only if the \textnormal{\Marpa} Earley item is a
$\Veim{y} = \tuple{\Vah{y}, \Vorig{x}}$
such that $\Vdr{x} \in \Vah{y}$.
A traditional Earley item, \Veimt{x}, corresponds to a
Marpa Earley item, \Veim{y}, if and only if
\Veim{y} corresponds to \Veimt{x}.
\end{definition}

\begin{definition}
A set of \type{EIM}'s is \dfn{consistent} with respect to
a set of \type{EIMT}'s,
if and only if each of the \type{EIM}'s in the first set
corresponds to at least one of the
\type{EIMT}'s in the second set.
A \textnormal{\Marpa} Earley set \EVtable{\Marpa}{i}
is \dfn{consistent} if and only if
all of its \type{EIM}'s correspond to
\type{EIMT}'s in
\EVtable{\Leo}{i}.
\end{definition}

\begin{definition}
A set of \type{EIM}'s is \dfn{complete} with respect to
a set of \type{EIMT}'s,
if and only if for every \type{EIMT} in the second set,
there is a corresponding \type{EIM} in the first set.
A \textnormal{\Marpa} Earley set \EVtable{\Marpa}{i}
is \dfn{complete} if and only if,
for every traditional Earley item in \EVtable{\Leo}{i},
there is a corresponding Earley item in
\EVtable{\Marpa}{i}.
\end{definition}

\begin{definition}
A \textnormal{\Marpa} Earley set is \dfn{correct}
if and only that \textnormal{\Marpa} Earley set is complete
and consistent.
\end{definition}

\subsection{About AHFA states}

Several facts from~\cite{AH2002}
will be heavily used in the following proofs.
For convenience, they are restated here.

\begin{observation}
Every dotted rule is an element of one
or more AHFA states, that is,
\begin{equation*}
\forall \, \Vdr{x} \, \exists \, \Vah{y} \; \mid \; \Vdr{x} \in \Vah{y}.
\end{equation*}
\end{observation}

\begin{observation}
\label{o:confirmed-AHFA-consistent}
AHFA confirmation is consistent with respect to the dotted rules.
That is,
for all \Vah{from}, \Vsym{t}, \Vah{to}, \Vdr{to} such that
\begin{equation*}
\begin{split}
& \GOTO(\Vah{from}, \Vsym{t}) = \Vah{to} \\
\qquad \qquad \land \quad & \Vdr{to} \in \Vah{to}, \\
\intertext{there exists \Vdr{from} such that}
& \Vdr{from} \in \Vah{from} \\
\qquad \qquad \land \quad & \Vsym{t} = \Postdot{\Vdr{from}} \\
\qquad \qquad \land \quad & \Next{\Vdr{from}} = \Vdr{to}. \\
\end{split}
\end{equation*}
\end{observation}

\begin{observation}
\label{o:confirmed-AHFA-complete}
AHFA confirmation is complete with respect to the dotted rules.
That is,
for all \Vah{from}, \Vsym{t}, \Vdr{from}, \Vdr{to} if
\begin{equation*}
\begin{split}
& \Vdr{from} \in \Vah{from} \\
\qquad \land \quad & \Postdot{\Vdr{from}} = \Vsym{t}, \\
\qquad \land \quad & \Next{\Vdr{from}} = \Vdr{to} \\
\intertext{then there exists \Vah{to} such that }
& \GOTO(\Vah{from}, \Vsym{t}) = \Vah{to}  \\
\qquad \land \quad & \Vdr{to} \in \Vah{to}. \\
\end{split}
\end{equation*}
\end{observation}

\begin{observation}
\label{o:predicted-AHFA-consistent}
AHFA prediction is consistent with respect to the dotted rules.
That is,
for all \Vah{from}, \Vah{to}, \Vdr{to} such that
\begin{equation*}
 \GOTO(\Vah{from}, \epsilon) = \Vah{to}
\, \land  \,
 \Vdr{to} \in \Vah{to},
\end{equation*}
there exists \Vdr{from} such that
\begin{equation*}
 \Vdr{from} \in \Vah{from}
\, \land  \,
 \Vdr{to} \in \Predict{\Vdr{from}}.
\end{equation*}
\end{observation}

\begin{observation}
\label{o:predicted-AHFA-complete}
AHFA prediction is complete with respect to the dotted rules.
That is,
for all \Vah{from}, \Vdr{from}, \Vdr{to}, if
\begin{equation*}
 \Vdr{from} \in \Vah{from}
\, \land \,
\Vdr{to} \in \Predict{\Vdr{from}},
\end{equation*}
then there exists \Vah{to} such that
\begin{equation*}
\Vdr{to} \in \Vah{to}
\, \land  \,
\GOTO(\Vah{from}, \epsilon) = \Vah{to}
\end{equation*}
\end{observation}

\section{\Marpa is correct}
\label{s:correct}

\subsection{\textnormal{\Marpa}'s Earley sets grow at worst linearly}

\begin{theorem}\label{t:es-count}
For a context-free grammar,
and a parse location \Vloc{i},
\begin{equation*}
\textup{
    $\bigsize{\EVtable{\Marpa}{i}} = \order{\var{i}}$.
}
\end{equation*}
\end{theorem}

\begin{proof}
\type{EIM}'s have the form $\tuple{\Vah{x}, \Vorig{x}}$.
\Vorig{x} is the origin of the \type{EIM},
which in \Marpa cannot be after the current
Earley set  at \Vloc{i},
so that
\begin{equation*}
0 \le \Vorig{x} \le \Vloc{i}.
\end{equation*}
The possibilities for \Vah{x} are finite,
since the number of AHFA states is a constant,
$\size{\Cfa}$,
which depends on \Cg{}.
Since duplicate \type{EIM}'s are never added to an Earley set,
the maximum size of Earley set \Vloc{i} is therefore
\begin{equation*}
\Vloc{i} \mult \size{\Cfa} = \order{\Vloc{i}}.\qedhere
\end{equation*}
\end{proof}

\subsection{\textnormal{\Marpa}'s Earley sets are correct}

\begin{theorem}\label{t:table-correct}
\textnormal{\Marpa}'s Earley sets are correct.
\end{theorem}

The proof
is by triple induction,
that is, induction with a depth down to 3 levels.
We number the levels of induction
0, 1 and 2,
starting with the outermost.
The level 0 induction is usually called the outer induction.
The level 1 induction is usually called the inner induction.
Level 2 induction is referred to by number.

The outer induction is on the Earley sets.
The induction variable is \Vloc{i},
and the outer induction hypothesis is that every Earley set
\EVtable{\Marpa}{j},
where $\Vloc{j} < \Vloc{i}$,
is complete and consistent,
and therefore correct.
For the outer induction step we need to show
that every Earley set
\EVtable{\Marpa}{k},
$\Vloc{k} \le \Vloc{i}$,
is complete and consistent.

Since we have correctness for the Earley sets at locations
less than \Vloc{i} by the outer induction hypothesis,
all we need to show for the step of the outer induction is that the Earley set
\EVtable{\Marpa}{i} is correct.
We leave it as an exercise to show, as the
basis of the outer induction, that
\EEtable{\Marpa}{0} is complete and consistent.

\begin{MYsloppyC}{3em}{2000}{5000}
To show the outer induction step, we show first
consistency, then completeness.
We show consistency by
an inner induction on the \Marpa operations.
The inner induction hypothesis is that
\EVtable{\Marpa}{i},
as so far built,
is consistent with respect to
\EVtable{\Leo}{i}.
\end{MYsloppyC}

As the basis of the inner induction,
an empty \Marpa Earley set is
consistent, trivially.
We show the step of the inner induction by cases:
\begin{itemize}
\item \Marpa{} scanning operations;
\item \Marpa{} reductions when there are no Leo reductions; and
\item \Marpa{}'s Leo reductions
\end{itemize}

\subsubsection{Marpa scanning is consistent}
\label{s:scan-consistent}

For \Marpa's scanning operation, we know
that the mainstem \type{EIM} is correct
by the outer induction hypothesis,
and that the token is correct
by the definitions in the preliminaries.
We know, from Section \ref{p:scan-op},
that at most two \type{EIM}'s will be added.
We now examine them in detail.

Let
\begin{equation*}
    \Vah{confirmed} = \GOTO(\Vah{mainstem}, \Vsym{token})
\end{equation*}
If $\Vah{confirmed} = \Lambda$,
the pseudocode of Section \ref{p:scan-op} shows
that we do nothing.
If we do nothing,
since \EVtable{\Marpa}{i} is consistent by the inner
induction hypothesis,
it remains consistent, trivially.

Otherwise, let
$\Veim{confirmed} = \tuple{\Vah{confirmed}, \Vloc{i}}$.
We see that \Veim{confirmed} is consistent with respect
to \EVtable{\Leo}{i},
by the definition of Earley scanning (Section~\ref{d:scan})
and Observation~\ref{o:confirmed-AHFA-consistent}.
Consistency is invariant under union,
and since \EVtable{\Marpa}{i} is consistent by the inner induction,
\EVtable{\Marpa}{i} remains consistent after
\Veim{confirmed} is added.

For predictions,
if $\Vah{confirmed} \ne \Lambda$, let
\begin{equation*}
\Vah{predicted} = \GOTO(\Vah{confirmed}, \epsilon)
\end{equation*}
If $\Vah{predicted} = \Lambda$,
the pseudocode of Section \ref{p:add-eim-pair} shows
that we do nothing.
If we do nothing,
since \EVtable{\Marpa}{i} is consistent by the inner
induction hypothesis,
it remains consistent, trivially.
Otherwise, let
\begin{equation*}
\Veim{predicted} = \tuple{\Vah{predicted}, \Vloc{i}}.
\end{equation*}
\begin{MYsloppyC}{10em}{8000}{5000}
We see that \Veim{predicted} is consistent with respect
to \EVtable{\Leo}{i},
by the definition of Earley prediction (Section~\ref{d:prediction}) and
\hbox{Observation~\ref{o:predicted-AHFA-consistent}}.
Consistency is invariant under union and,
since \EVtable{\Marpa}{i} is consistent by the inner induction,
\EVtable{\Marpa}{i} remains consistent after
\Veim{predicted} is added.
\end{MYsloppyC}

\subsubsection{Earley reduction is consistent}
\label{s:reduction-consistent}

Next,
we show that \Marpa{}'s reduction operation
is consistent,
in the case where there is no Leo reduction.
The reduction will be the result of the two \type{EIM}'s
of a confluence, call
them \Veim{mainstem} and \Veim{tributary}.
\Veim{mainstem} will be correct by the outer induction
hypothesis
and \Veim{tributary}
will be consistent by the inner induction hypothesis.
From \Veim{tributary}, we will find zero or more transition
symbols, \Vsym{lhs}.
From this point,  the argument is very similar to
that for the case of the scanning operation.

Let
\begin{equation*}
\Vah{confirmed} = \GOTO(\Vah{mainstem}, \Vsym{lhs})
\end{equation*}
If $\Vah{confirmed} = \Lambda$, we do nothing,
and \EVtable{\Marpa}{i} remains consistent, trivially.
Otherwise, let
\begin{equation*}
\Veim{confirmed} = \tuple{\Vah{confirmed}, \Vloc{i}}.
\end{equation*}
We see that \Veim{confirmed} is consistent with respect
to \EVtable{\Leo}{i}
by the definition of Earley reduction (Section~\ref{s:reduction}),
and Observation~\ref{o:confirmed-AHFA-consistent}.
By the invariance of consistency under union,
\EVtable{\Marpa}{i} remains consistent after
\Veim{confirmed} is added.

For predictions, the argument exactly repeats that of
Section \ref{s:scan-consistent}.
\EVtable{\Marpa}{i} remains consistent,
whether or not a \Veim{predicted} is added.

\subsubsection{Leo reduction is consistent}
\label{s:leo-consistent}

\begin{MYsloppy}
We now show consistency for \Marpa{}'s
reduction operation,
in the case where there is a Leo reduction.
If there is a Leo reduction, it is signaled by the
presence of \Vlim{mainstem},
\end{MYsloppy}
\begin{equation*}
\Vlim{mainstem} = \tuple{ \Vah{top}, \Vsym{lhs}, \Vorig{top} }
\end{equation*}
in the Earley set where we would look
for the \Veim{mainstem}.
We treat
the logic to create \Vlim{mainstem} as a matter of memoization
of the previous Earley sets,
and its correctness follows from
the outer induction hypothesis.

As the result of a Leo reduction,
\Leo{} will add
$\tuple{\Vdr{top}, \Vorig{top}}$
to \EVtable{\Leo}{j}.
Because the \Marpa{} \type{LIM} is correct,
using Observations \ref{o:confirmed-AHFA-consistent}
and \ref{o:confirmed-AHFA-complete}
and Theorem \ref{t:leo-singleton},
we see that \Vah{top} is the singleton set
$\set{ \Vdr{top} }$.
From Section \ref{p:leo-op}, we see
that, as the result of the Leo reduction,
\Marpa{} will add
\begin{equation*}
\Veim{leo} = \tuple{\Vah{top}, \Vorig{top}}
\end{equation*}
to \EVtable{\Marpa}{j}.
The consistency of \Veim{leo} follows from the definition
of \type{EIM} consistency.
The consistency of
\EVtable{\Marpa}{i},
once \Veim{leo} is added,
follows by the invariance
of consistency under union.

\subsubsection{Marpa's Earley sets are consistent}
\label{s:sets-consistent}

Sections
\ref{s:scan-consistent},
\ref{s:reduction-consistent}
and
\ref{s:leo-consistent}
show the cases for the step of the inner induction,
which shows the induction.
It was the purpose of the inner induction to show
that consistency of \EVtable{\Marpa}{i} is invariant
under \Marpa's operations.

\subsubsection{The inner induction for completeness}

It remains to show that,
when \Marpa's operations are run as described
in the pseudocode of Section \ref{s:pseudocode},
that
\EVtable{\Marpa}{i} is complete.
To do this,
we show that
at least one \type{EIM} in \EVtable{\Marpa}{i}
corresponds to every \type{EIMT} in
\EVtable{\Leo}{i}.
We will proceed by cases,
where the cases are \Leo{} operations.
For every operation that \Leo{} would perform,
we show that
\Marpa{} performs an operation that
produces a corresponding Earley item.
Our cases for the operations of \Leo{} are
Earley scanning operations;
Earley reductions;
Leo reductions;
and Earley predictions.

\subsubsection{Scanning is complete}
\label{s:scan-complete}

For scanning, the \Marpa pseudocode
(Algorithm \ref{alg:scan-pass} on page \pageref{alg:scan-pass})
shows
that a scan is attempted for every
pair
\begin{equation*}
\tuple{\Veim{mainstem}, \Vsym{token}},
\end{equation*}
where \Veim{mainstem} is an \type{EIM} in the previous
Earley set,
and \Vsym{token} is the token scanned at \Vloc{i}.
(Algorithm \ref{alg:scan-pass} actually finds
\Veim{mainstem} in a set
returned by $\mymathop{transitions}()$.
This is a memoization for efficiency
and we will ignore it.)

By the preliminary definitions, we know that \Vsym{token}
is the same in both \Earley{} and \Leo.
By the outer induction hypothesis we know that,
for every traditional Earley item in the previous
Earley set,
there is at least one corresponding \Marpa Earley item.
Therefore, \Marpa{} performs its scan operation on a complete set
of mainstems.

Comparing the \Marpa pseudocode (Section \ref{p:scan-op}),
with the Earley scanning operation (Section \ref{d:scan})
and using
Observations~\ref{o:confirmed-AHFA-complete}
and \ref{o:predicted-AHFA-complete},
we see that an Earley item will be added to
\EVtable{\Marpa}{i} corresponding to every scanned Earley item
of \EVtable{\Leo}{i}.
We also see,
from the pseudocode of Section \ref{p:add-eim-pair},
that the \Marpa{} scanning operation will
add to \EVtable{\Marpa}{i}
an Earley item for
every prediction that results from
a scanned Earley item in \EVtable{\Leo}{i}.

\subsubsection{Earley reduction is complete}
\label{s:reduction-complete}

We now examine Earley reduction,
under the assumption that there is
no Leo transition.
The \Marpa pseudocode shows that the Earley items
in \EVtable{\Marpa}{i}
are traversed in a single pass for reduction.

To show that we traverse a complete and consistent
series of tributary Earley items,
we stipulate that
the Earley set is an ordered set,
and that new Earley items are added at the end.
From Theorem \ref{t:es-count}, we know
that
the number of Earley items is finite,
so a traversal of them must terminate.

Consider, for the purposes of the level 2 induction,
the reductions of \Leo{} to occur in generations.
Let the scanned Earley items be generation 0.
An \type{EIMT} produced by a reduction is generation $\var{n} + 1$
if its tributary Earley item was in generation \var{n}.
Predicted Earley items do not need to be assigned generations.
In \Marpa grammars they can never contain completions,
and therefore can never act as the tributary of a reduction.

The level 2 induction is on generations.
In Section \ref{s:scan-complete},
we showed that generation 0 is complete --
it contains Earley items
corresponding to all of the generation 0 \type{EIMT}'s of \Leo.
This is the basis of the level 2 induction.

The generation variable for the level 2 induction is \V{g}.
The induction hypothesis for the step of level 2 induction
is that for some \var{g},
the Earley items of \EVtable{\Marpa}{i}
for the generations prior to \V{g}
are correct (that is, complete and consistent).
For the step we need to show that
the Earley items of \EVtable{\Marpa}{i} for the generation up to \V{g}
are correct.
Since we have the correctness of the generation prior to \V{g}
by the induction hypothesis,
all that we need to show for the step will be that
the Earley items of \EVtable{\Marpa}{i} for generation \V{g}
are correct.

From Section \ref{s:sets-consistent},
we know that all Earley items in \Marpa's sets are consistent.
Therefore, to show correctness, we have only to show completeness.

Since we stipulated that \Marpa{} adds Earley items
at the end of each set,
we know that they occur in generation order.
Therefore \Marpa{},
when creating Earley items of generation $\var{n}+1$
while traversing \EVtable{\Marpa}{i},
can rely
on the level 2 induction hypothesis for
the completeness of Earley items
in generation \var{n}.

Let
$$\Veim{working} \in \Ves{i}$$
be the Earley item
currently being considered as a potential tributary for
an Earley reduction operation.
From the pseudocode, we see
that reductions are attempted for every
pair \Veim{mainstem}, \Veim{working}.
(Again, $\mymathop{transitions}()$ is ignored
as a memoization.)
By the outer induction hypothesis we know that,
for every traditional Earley item in the previous
Earley set,
there is at least one corresponding \Marpa Earley item.
We see from the pseudocode, therefore,
that for each \Veim{working}
that \Marpa{} performs its reduction operation on a complete set
of correct mainstems.
Therefore \Marpa{} performs its reduction operations on a
complete set of confluences.

Comparing the \Marpa pseudocode (Section \ref{p:reduction-op})
with the Earley reduction operation (Section \ref{s:reduction})
and using
Observations~\ref{o:confirmed-AHFA-complete}
and \ref{o:predicted-AHFA-complete},
we see that a Earley reduction result of
generation $\var{n}+1$
will be added to
\EVtable{\Marpa}{i} corresponding to every Earley reduction result
in generation $\var{n}+1$
of \EVtable{\Leo}{i},
as well as one corresponding
to every prediction that results from
an Earley reduction result
of generation $\var{n}+1$ in \EVtable{\Leo}{i}.
This shows the level 2 induction
and the case of reduction completeness.

\subsubsection{Leo reduction is complete}
\label{s:leo-complete}

\begin{MYsloppy}
We now show completeness for \Marpa{}'s reduction operation,
in the case where there is a Leo reduction.
In Section \ref{s:leo-consistent},
we found that where \Leo{} would create
the \type{EIMT}
\end{MYsloppy}
$$\tuple{\Vdr{top}, \Vorig{top}},$$
\Marpa adds
$$\tuple{\Vah{top}, \Vorig{top}}$$
such that $\Vdr{top} \in \Vah{top}$.
Since \Vdr{top} is a completed rule,
there are no predictions.
This shows the case immediately,
by the definition of completeness.

\subsubsection{Prediction is complete}
\label{s:prediction-complete}

\begin{MYsloppy}
Predictions result only from items in the same Earley set.
In Sections \ref{s:scan-complete},
\ref{s:reduction-complete}
and \ref{s:leo-complete},
we showed that,
for every prediction that would result
from an item added to \EVtable{\Leo}{i},
a corresponding prediction
was added to \EVtable{\Marpa}{i}.
\end{MYsloppy}

\subsubsection{Finishing the proof}
Having shown the cases in Sections
\ref{s:scan-complete},
\ref{s:reduction-complete},
\ref{s:leo-complete} and
\ref{s:prediction-complete},
we know that Earley set
\EVtable{\Marpa}{i} is complete.
In Section \ref{s:sets-consistent}
we showed that \EVtable{\Marpa}{i} is consistent.
It follows that \EVtable{\Marpa}{i} is correct,
which is the step of the outer induction.
Having shown its step, we have the outer induction,
and the theorem.
\qedsymbol

\subsection{\textnormal{\Marpa} is correct}

We are now in a position to show that \Marpa is correct.

\begin{theorem}
\textup{ $\myL{\Marpa,\Cg} = \myL{\Cg}$ }
\end{theorem}

\begin{proof}
From Theorem \ref{t:table-correct},
we know that
\begin{equation*}
\tuple{\Vdr{accept},0} \in \EVtable{\Leo}{\Vsize{w}}
\end{equation*}
if and only there is a
\begin{equation*}
\tuple{\Vah{accept},0} \in \EVtable{\Marpa}{\Vsize{w}}
\end{equation*}
such that $\Vdr{accept} \in \Vah{accept}$.
From the acceptance criteria in the \Leo{} definitions
and the \Marpa{} pseudocode,
it follows that
\begin{equation*}
\myL{\Marpa,\Cg} = \myL{\Leo,\Cg}.
\end{equation*}
By Theorem 4.1 in~\cite{Leo1991}, we know that
\begin{equation*}
\myL{\Leo,\Cg} = \myL{\Cg}.
\end{equation*}
The theorem follows from
the previous two equalities.
\end{proof}

\section{\Marpa recognizer complexity}
\label{s:complexity}

\subsection{Complexity of each Earley item}

For the complexity proofs,
we consider only \Marpa grammars without nulling
symbols.
We showed that this rewrite
is without loss of generality
in Section \ref{s:nulling},
when we examined correctness.
For complexity we must also show that
the rewrite and its reversal can be done
in amortized \Oc{} time and space
per Earley item.

\begin{lemma}\label{l:nulling-rewrite}
All time and space required
to rewrite the grammar to eliminate nulling
symbols, and to restore those rules afterwards
in the Earley sets,
can be allocated
to the Earley items
in such a way that each Earley item
requires \Oc{} time and space.
\end{lemma}

\begin{proof}
The time and space used in the rewrite is a constant
that depends on the grammar,
and is charged to the parse.
The reversal of the rewrite can be
done in a loop over the Earley items,
which will have time and space costs
per Earley item,
plus a fixed overhead.
The fixed overhead is \Oc{}
and is charged to the parse.
The time and space per Earley item
is \Oc{}
because the number of
rules into which another rule must be rewritten,
and therefore the number of Earley items
into which another Earley item must be rewritten,
is a constant that depends
on the grammar.
\end{proof}

\begin{theorem}\label{t:O1-time-per-eim}
All time in \textnormal{\Marpa} can be allocated
to the Earley items,
in such a way that each Earley item,
and each attempt to
add a duplicate Earley item,
requires \Oc{} time.
\end{theorem}

\begin{theorem}\label{t:O1-space-per-eim}
All space in \textnormal{\Marpa} can be allocated
to the Earley items,
in such a way that each Earley item
requires \Oc{} space and,
if confluences are not considered,
each attempt to add a duplicate
Earley item adds no additional space.
\end{theorem}

\begin{theorem}\label{t:O1-confluences-per-eim}
If confluences are considered,
all space in \textnormal{\Marpa} can be allocated
to the Earley items
in such a way that each Earley item
and each attempt to
add a duplicate Earley item
requires \Oc{} space.
\end{theorem}

\begin{proof}[Proof of Theorems
\ref{t:O1-time-per-eim},
\ref{t:O1-space-per-eim},
and \ref{t:O1-confluences-per-eim}]
These theorems follows from the observations
in Section \ref{s:pseudocode}
and from Lemma \ref{l:nulling-rewrite}.
\end{proof}

\subsection{Duplicate dotted rules}

The same complexity results apply to \Marpa{} as to \Leo,
and the proofs are very similar.
\Leo's complexity results~\cite{Leo1991}
are based on charging
resource to Earley items,
as were the results
in Earley's paper~\cite{Earley1970}.
But both assume that there is one dotted rule
per Earley item,
which is not the case with \Marpa.

\Marpa's Earley items group dotted rules into AHFA
states, but this is not a partitioning in the strict
sense -- dotted rules can fall into more than one AHFA
state.
This is an optimization,
in that it allows dotted rules,
if they often occur together,
to be grouped together aggressively.
But it opens up the possibility
that, in cases where \Earley{} and \Leo{} disposed
of a dotted rule once and for all,
\Marpa{} might have to deal with it multiple times.
\Marpa's duplicate rules
do not change the complexity results,
although showing this requires some additional
theoretical apparatus,
which this section contains.

\begin{theorem}\label{t:marpa-O-leo}
\begin{equation*}
\textup{
    $\Rtablesize{\Marpa} < \var{c} \mult \Rtablesize{\Leo}$,
}
\end{equation*}
where \var{c} is a constant that depends on the grammar.
\end{theorem}

\begin{proof}
We know from Theorem \ref{t:table-correct}
that every \Marpa Earley item corresponds to one of
\Leo's traditional Earley items.
If an \type{EIM} corresponds to an \type{EIMT},
the AHFA state of the \type{EIM} contains the
\type{EIMT}'s dotted rule,
while their origins are identical.
Even in the worst case, a dotted rule cannot
appear in every AHFA state,
so that
the number of \Marpa items corresponding to a single
traditional Earley item must be less
than $\size{\Cfa}$.
Therefore,
\begin{equation*}
    \Rtablesize{\Marpa} < \size{\Cfa} \mult \Rtablesize{\Leo}\qedhere
\end{equation*}
\end{proof}

Earley~\cite{Earley1970} shows that,
for unambiguous grammars,
every attempt to add
an Earley item will actually add one.
In other words, there will be no attempts to
add duplicate Earley items.
Earley's proof shows that for each attempt
to add a duplicate,
the causation must be different --
that the confluences causing the attempt
will differ in either their mainstem
or their tributary.
Multiple confluences for an Earley item
would mean multiple derivations
for the sentential form that it represents.
That in turn would mean that
the grammar is ambiguous,
contrary to assumption.

In \Marpa, there is an slight complication.
A dotted rule can occur in more than one AHFA
state.
Because of that,
it is possible that two of \Marpa's
operations to add an \type{EIM}
will represent identical Earley confluences,
and therefore will be
consistent with an unambiguous grammar.
Dealing with this complication requires us
to prove a result that is weaker than that of~\cite{Earley1970},
but that is
still sufficient to produce the same complexity results.

\begin{theorem}\label{t:tries-O-eims}
For an unambiguous grammar,
the number of attempts to add
Earley items will be less than or equal to
\begin{equation*}
\textup{
    $\var{c} \mult \Rtablesize{\Marpa}$,
}
\end{equation*}
where \var{c} is a constant
that depends on the grammar.
\end{theorem}

\begin{proof}
Let \var{initial-tries} be the number of attempts to add the initial item to
the Earley sets.
For Earley set 0, it is clear from the pseudocode
that there will be no attempts to add duplicate \type{EIM}'s:
\begin{equation*}
\var{initial-tries} = \bigsize{\Vtable{0}}
\end{equation*}

Let \var{leo-tries} be the number of attempted Leo reductions in
Earley set \Vloc{j}.
For Leo reduction,
we note that by its definition,
duplicate attempts at Leo reduction cannot occur.
Let \var{max-AHFA} be the maximum number of
dotted rules in any AHFA state.
From the pseudo-code of Sections \ref{p:reduce-one-lhs}
and \ref{p:leo-op},
we know there will be at most one Leo reduction for
each each dotted rule in the current Earley set,
\Vloc{j}.
\begin{equation*}
\var{leo-tries} \le \var{max-AHFA} \mult \bigsize{\Vtable{j}}
\end{equation*}

Let \var{scan-tries} be the number of attempted scan operations in
Earley set \Vloc{j}.
Marpa attempts a scan operation,
in the worst case,
once for every \type{EIM} in the Earley set
at $\Vloc{j} \subtract 1$.
Therefore, the number of attempts
to add scans
must be less than equal to \bigsize{\Etable{\var{j} \subtract 1}},
the number
of actual Earley items at
$\Vloc{j} \subtract 1$.
\begin{equation*}
\var{scan-tries} \le \bigsize{\Etable{\var{j} \subtract 1}}
\end{equation*}

Let \var{predict-tries} be the number of attempted predictions in
Earley set \Vloc{j}.
\Marpa{} includes prediction
in its scan and reduction operations,
and the number of attempts to add duplicate predicted \type{EIM}'s
must be less than or equal
to the number of attempts
to add duplicate confirmed \type{EIM}'s
in the scan and reduction operations.
\begin{equation*}
\var{predict-tries} \le \var{reduction-tries} + \var{scan-tries}
\end{equation*}

The final and most complicated case is Earley reduction.
Recall that \Ves{j} is the current Earley set.
Consider the number of reductions attempted.
\Marpa{} attempts to add an Earley reduction result
once for every triple
\begin{equation*}
\tuple{\Veim{mainstem}, \Vsym{transition}, \Veim{tributary}}.
\end{equation*}
where
\begin{equation*}
\begin{split}
& \Veim{tributary} = \tuple{ \Vah{tributary}, \Vloc{tributary-origin} }  \\
\land \quad & \Vdr{tributary} \in \Vah{tributary} \\
 \land \quad & \Vsym{transition} = \LHS{\Vdr{tributary}}. \\
\end{split}
\end{equation*}

We now put an upper bound on number of possible values of this triple.
The number of possibilities for \Vsym{transition} is clearly at most
\size{\var{symbols}},
the number of symbols in \Cg{}.
We have $\Veim{tributary} \in \Ves{j}$,
and therefore there are at most
$\bigsize{\Etable{\V{j}}}$ choices for \Veim{tributary}.

We can show that the number of possible choices of
\Veim{mainstem} is at most
the number of AHFA states, \Vsize{fa}, by a reductio.
Suppose, for the reductio,
there were more than \Vsize{fa} possible choices of \Veim{mainstem}.
Then there are two possible choices of \Veim{mainstem} with
the same AHFA state.
Call these \Veim{choice1} and \Veim{choice2}.
We know, by the definition of Earley reduction, that
$\Veim{mainstem} \in \Ves{j}$,
and therefore we have
$\Veim{choice1} \in \Ves{j}$ and
$\Veim{choice2} \in \Ves{j}$.
Since all \type{EIM}'s in an Earley set must differ,
and
\Veim{choice1} and \Veim{choice2} both have the same
AHFA state,
they must differ in their origin.
But two different origins would produce two different derivations for the
reduction, which would mean that the parse was ambiguous.
This is contrary to the assumption for the theorem
that the grammar is unambiguous.
This shows the reductio
and that the number of choices for \Veim{mainstem},
compatible with \Vorig{tributary}, is as most \Vsize{fa}.

\begin{MYsloppyB}{3em}{2000}
Collecting the results, we see that
the number of possible choices for
each \Veim{tributary} is
\begin{gather*}
\Vsize{fa} \mult \Vsize{symbols} \mult \bigsize{\Etable{\V{j}}}, \\
\text{where } \Vsize{fa}
  \text{is the number of possible choices for \Veim{mainstem}}, \\
\Vsize{symbols}
  \text{is the number of possible choices for \Vsym{transition}, and} \\
\bigsize{\Etable{\V{j}}}
  \text{is the number of possible choices for \Veim{tributary}.}
\end{gather*}
\end{MYsloppyB}

The number of reduction attempts will therefore be at most
\begin{equation*}
\var{reduction-tries} \leq \Vsize{fa} \mult \Vsize{symbols} \mult \bigsize{\Etable{\V{j}}}.
\end{equation*}

Summing
\begin{multline*}
\var{tries} =
\var{scan-tries} +
\var{leo-tries} + \\
\var{predict-tries} +
\var{reduction-tries} +
\var{initial-tries},
\end{multline*}
we have,
where $\var{n} = \Vsize{\Cw}$,
the size of the input,
\begin{equation*}
\begin{alignedat}{2}
& \bigsize{\Vtable{0}} & \quad &
\qquad \text{initial \type{EIM}'s} \\
+ \; & \sum\limits_{j=0}^{n}{
\var{max-AHFA} \mult \bigsize{\Vtable{j}}
} &&
\qquad \text{\type{LIM}'s} \\
+ \; & 2 \mult \sum\limits_{j=1}^{n}{
\bigsize{\Etable{\var{j} \subtract 1}}
} &&
\qquad \text{scanned \type{EIM}'s} \\
+ \; & 2 \mult \sum\limits_{j=0}^{n}{\Vsize{fa} \mult \Vsize{symbols} \mult \bigsize{\Etable{\V{j}}}} &&
\qquad \text{reduction \type{EIM}'s}.
\end{alignedat}
\end{equation*}
In this summation,
\var{prediction-tries} was accounted for by counting the scanned and predicted
\type{EIM} attempts twice.
Since \var{max-AHFA} and \Vsize{symbols} are both constants
that depend only on \Cg{},
if we collect the terms of the summation,
we will find a constant \var{c}
such that,
where \var{c} is a constant that depends on \Cg{},
\begin{equation*}
\var{tries} \leq \var{c} \mult \sum\limits_{j=0}^{n}{\bigsize{\EEtable{\Marpa}{\V{j}}}}.
\end{equation*}
Changing the index from \V{j} to \V{i},
and abbreviating the count of all Earley items according to the convention
of \eqref{eq:abbr-all-eims-count} on
page \pageref{eq:abbr-all-eims-count}, we have
\begin{equation*}
  \V{tries} \leq \V{c} \mult \Rtablesize{\Marpa}.\qedhere
\end{equation*}
\end{proof}

As a reminder,
we follow tradition by
stating complexity results in terms of \var{n},
setting $\var{n} = \Vsize{\Cw}$,
the length of the input.

\begin{theorem}\label{t:eim-count}
For a context-free grammar,
\begin{equation*}
\textup{
    $\Rtablesize{\Marpa} = \order{\var{n}^2}$.
}
\end{equation*}
\end{theorem}

\begin{proof}
By Theorem \ref{t:es-count},
the size of the Earley set at \Vloc{i}
is $\order{\var{i}}$.
Summing over the length of the input,
$\Vsize{\Cw} = \var{n}$,
the number of \type{EIM}'s in all of \Marpa's Earley sets
is
\begin{equation*}
\sum\limits_{\Vloc{i}=0}^{\var{n}}{\order{\var{i}}}
= \order{\var{n}^2}.\qedhere
\end{equation*}
\end{proof}

\begin{theorem}\label{t:ambiguous-tries}
For a context-free grammar,
the number of attempts to add
Earley items is $\order{\var{n}^3}$.
\end{theorem}

\begin{proof}
Reexamining the proof of Theorem \ref{t:tries-O-eims},
we see that the only bound that required
the assumption that \Cg{} was unambiguous
was \var{reduction-tries},
the count of the number of attempts to
add Earley reductions.
Let \var{other-tries}
be attempts to add \type{EIM}'s other than
as the result of Earley reductions.
By Theorem \ref{t:eim-count},
\begin{equation*}
\Rtablesize{\Marpa} = \order{\var{n}^2},
\end{equation*}
and by Theorem \ref{t:tries-O-eims},
\begin{equation*}
\var{other-tries} \le \var{c} \mult \Rtablesize{\Marpa},
\end{equation*}
so that
$\var{other-tries} = \order{\var{n}^2}$.

Looking again at \var{reduction-tries}
for the case of ambiguous grammars,
we need to look again at the triple
\begin{equation*}
\tuple{\Veim{mainstem}, \Vsym{transition}, \Veim{tributary}}.
\end{equation*}
We did not use the fact that the grammar was unambiguous in counting
the possibilities for \Vsym{transition} or \Veim{tributary}, but
we did make use of it in determining the count of possibilities
for \Veim{mainstem}.
We still know that
\begin{equation*}
\Veim{mainstem} \in \Ves{tributary-origin},
\end{equation*}
where
\Vloc{tributary-origin} is the origin of \Veim{tributary}.
Worst case, every
$$ \type{EIM} \in \Ves{tributary-origin} $$
is a possible
match, so that
the number of possibilities for \Veim{mainstem} now grows to
\bigsize{\Vtable{tributary-origin}}, and
\begin{equation*}
\begin{gathered}
\var{reduction-tries} = \\
    \bigsize{\Vtable{tributary-origin}} \mult \Vsize{symbols} \mult \bigsize{\Etable{\V{j}}}.
\end{gathered}
\end{equation*}

We know that $\var{tributary-origin} \le \var{j}$,
so that by Theorem \ref{t:es-count},
\begin{equation*}
\size{\Vtable{tributary-origin}} \mult \size{\Etable{\V{j}}} = \order{\var{j}^2}.
\end{equation*}
Adding \var{other-tries}
and summing over the Earley sets,
we have
\begin{equation*}
\order{\var{n}^2} +
\! \sum\limits_{\Vloc{j}=0}^{n}{\order{\var{j}^2}} = \order{\var{n}^3}.
\qedhere
\end{equation*}
\end{proof}

\begin{theorem}\label{t:leo-right-recursion}
Either
a right derivation has a step
that uses a right recursive rule,
or it has length is at most \var{c},
where \var{c} is a constant which depends
on the grammar.
\end{theorem}

\begin{proof}
Let the constant \var{c} be the number
of symbols.
Assume, for a reductio, that a right derivation
expands to a
Leo sequence of length
$\var{c}+1$, but that none of its steps uses a right recursive rule.

Because it is of length $\var{c}+1$,
the same symbol must appear twice as the rightmost symbol of
a derivation step.
(Since for the purposes of these
complexity results we ignore nulling symbols,
the rightmost symbol of a string will also be its rightmost
non-nulling symbol.)
So part of the rightmost derivation must take the form
\begin{equation*}
\Vstr{earlier-prefix} \cat \Vsym{A} \deplus \Vstr{later-prefix} \cat \Vsym{A}.
\end{equation*}
But the first step of this derivation sequence must use a rule of the
form
\begin{equation*}
\Vsym{A} \de \Vstr{rhs-prefix} \cat \Vsym{rightmost},
\end{equation*}
where $\Vsym{rightmost} \deplus \Vsym{A}$.
Such a rule is right recursive by definition.
This is contrary to the assumption for the reductio.
We therefore conclude that the length of a right derivation
must be less than or equal to \var{c},
unless at least one step of that derivation uses a right recursive rule.
\end{proof}

\subsection{The complexity results}
We are now in a position to show
specific time and space complexity results.

\begin{theorem}
For every LR-regular grammar,
\textnormal{\Marpa} runs in $\On{}$ time and space.
\end{theorem}

\begin{proof}
By Theorem 4.6 in~\cite[p. 173]{Leo1991},
the number of traditional Earley items produced by
\Leo{} when parsing input \Cw{} with an LR-regular grammar \Cg{} is
\begin{equation*}
\order{\Vsize{\Cw}} = \order{\var{n}}.
\end{equation*}
\Marpa{} may produce more Earley items than \Leo{}
for two reasons:
First, \Marpa{} does not apply Leo memoization to Leo sequences
which do not contain right recursion.
Second, \Marpa{}'s Earley items group dotted rules into
states and this has the potential to increase the number
of Earley items.

By theorem \ref{t:leo-singleton},
the definition of an \type{EIMT},
and the construction of a Leo sequence,
it can be seen that a Leo sequence
corresponds step-for-step with a
right derivation.
It can therefore be seen that
the number of \type{EIMT}'s in the Leo sequence
and the number of right derivation steps
in its corresponding right derivation
will be the same.

Consider one \type{EIMT} that is memoized in \Leo{}.
By theorem \ref{t:leo-singleton} it corresponds to
a single dotted rule, and therefore a single rule.
If not memoized because it is not a right recursion,
this \type{EIMT} will be expanded to a sequence
of \type{EIMT}'s.
How long will this sequence of non-memoized \type{EIMT}'s
be, if we still continue to memoize \type{EIMT}'s
which correspond to right recursive rules?
The \type{EIMT} sequence, which was formerly a memoized Leo sequence,
will correspond to a right
derivation that does not include
any steps that use right recursive rules.
By Theorem \ref{t:leo-right-recursion},
such a
right derivation can be
of length at most \var{c1},
where \var{c1} is a constant that depends on \Cg{}.
As noted, this right derivation has
the same length as its corresponding \type{EIMT} sequence,
so that each \type{EIMT} not memoized in \Marpa{} will expand
to at most \var{c1} \type{EIMT}'s.

By Theorem \ref{t:marpa-O-leo},
when \type{EIMT}'s are replaced with \type{EIM}'s,
the number of \type{EIM}'s \Marpa{} requires is at worst,
$\var{c2}$ times the number of \type{EIMT}'s,
where \var{c2} is a constant that depends on \Cg{}.
Therefore the number of \type{EIM}'s per Earley set
for an LR-regular grammar in a \Marpa{} parse
is less than
\begin{equation*}
    \var{c1} \mult \var{c2} \mult \order{\var{n}} = \order{\var{n}}.
\end{equation*}

LR-regular grammar are unambiguous, so that
by Theorem \ref{t:tries-O-eims},
the number of attempts that \Marpa{} will make to add
\type{EIM}'s is less than or equal to
\var{c3} times the number of \type{EIM}'s,
where \var{c3} is a constant that depends on \Cg{}.
Therefore,
by Theorems \ref{t:O1-time-per-eim}
and \ref{t:O1-confluences-per-eim},
the time and space complexity of \Marpa{} for LR-regular
grammars is
\begin{equation*}
    \var{c3} \mult \order{\var{n}}
    = \order{\var{n}}.\qedhere
\end{equation*}
\end{proof}

\begin{theorem}
For every unambiguous grammar,
\textnormal{\Marpa} runs in $\order{n^2}$ time and space.
\end{theorem}

\begin{proof}
Let \V{g} be an
unambiguous grammar.
By Theorem \ref{t:tries-O-eims},
and Theorem \ref{t:eim-count},
the number of attempts that \Marpa{} will make to add
\type{EIM}'s is
\begin{equation*}
\var{c} \mult \order{\var{n}^2},
\end{equation*}
where \var{c} is a constant that depends on \Cg{}.
Therefore,
by Theorems \ref{t:O1-time-per-eim}
and \ref{t:O1-confluences-per-eim},
the time and space complexity of \Marpa{}
for unambiguous grammars is \order{\var{n}^2}.
\end{proof}

\begin{theorem}
For every context-free grammar,
\textnormal{\Marpa} runs in $\order{\var{n}^3}$ time.
\end{theorem}

\begin{proof}
By Theorem \ref{t:O1-time-per-eim},
and Theorem \ref{t:ambiguous-tries}.
\end{proof}

\begin{theorem}\label{t:cfg-space}
For every context-free grammar,
\textnormal{\Marpa} runs in $\order{\var{n}^2}$ space,
if it does not track confluences.
\end{theorem}

\begin{proof}
By Theorem \ref{t:O1-space-per-eim}
and Theorem \ref{t:eim-count}.
\end{proof}

Traditionally only the space result stated for a parsing algorithm
is that
without confluences, as in \ref{t:cfg-space}.
This is sufficiently relevant
if the parser is only used as a recognizer.
In practice, however,
algorithms like \Marpa{}
are typically used in anticipation
of an evaluation phase,
for which confluences are necessary.

\begin{theorem}
For every context-free grammar,
\textnormal{\Marpa} runs in $\order{\var{n}^3}$ space,
including the space for tracking confluences.
\end{theorem}

\begin{proof}
By Theorem \ref{t:O1-confluences-per-eim},
and Theorem \ref{t:ambiguous-tries}.
\end{proof}

\section{The \Marpa input model}
\label{s:input}

In this paper,
up to this point,
the traditional input stream model
has been assumed.
As implemented in~\cite{Marpa-R2},
Marpa generalizes the idea of
input streams beyond the traditional
model.

Marpa's generalized input model
replaces the input \Cw{}
with a set of tokens,
\var{tokens},
whose elements are triples of symbol,
start location and length:
\begin{equation*}
    \tuple{\Vsym{t}, \Vloc{start}, \var{length}}
\end{equation*}
such that
$\var{length} \ge 1$
and
$\Vloc{start} \ge 0$.
The size of the input, \size{\Cw},
is the maximum over
\var{tokens} of $\Vloc{start}+\var{length}$.

Multiple tokens can start at a single location.
(This is how \Marpa{} supports ambiguous tokens.)
The variable-length,
ambiguous and overlapping tokens
of \Marpa{}
bend the conceptual framework of ``parse location''
beyond its breaking point,
and a new term for parse location is needed.
Start and end of tokens are described in terms
of \dfn{earleme} locations,
or simply \dfn{earlemes}.
Token length is also measured in earlemes.

Like standard parse locations, earlemes start at 0,
and run up to \size{\Cw}.
Unlike standard parse locations,
there is not necessarily a token ``at'' any particular earleme.
(A token is considered to be ``at an earleme'' if it ends there,
so that there is never a token ``at'' earleme 0.)
In fact,
there may be earlemes at which no token either starts or ends,
although for the parse to succeed, such an earleme would have to be
properly inside at least one token.
Here ``properly inside'' means after the token's start earleme
and before the token's end earleme.

In the \Marpa input stream, tokens
may interweave and overlap freely,
but gaps are not allowed.
That is, for all \Vloc{i} such
that $0 \le \Vloc{i} < \size{\Cw}$,
there must exist
\begin{equation*}
	 \var{token} = \tuple{\Vsym{t}, \Vloc{start}, \var{length}}
\end{equation*}
such that
\begin{gather*}
	 \var{token} \in \var{tokens} \quad \text{and} \\
	 \Vloc{start} \le \Vloc{i} < \Vloc{start}+\var{length}.
\end{gather*}

The intent of \Marpa's generalized input model is to allow
users to define alternative input models for special
applications.
An example that arises in current practice is natural
language, features of which are most
naturally expressed with ambiguous tokens.
The traditional input stream can be seen as the special case of
the \Marpa input model where
for all \Vsym{x}, \Vsym{y}, \Vloc{x}, \Vloc{y},
\var{xlength}, \var{ylength},
if we have both of
\begin{align*}
    \tuple{\Vsym{x}, \Vloc{x}, \var{xlength}} & \in \var{tokens} \quad \text{and} \\
    \tuple{\Vsym{y}, \Vloc{y}, \var{ylength}} & \in \var{tokens},
\end{align*}
then we have both of
\begin{gather*}
\var{xlength} = \var{ylength} = 1 \quad \text{and} \\
     \Vloc{x} = \Vloc{y} \implies \Vsym{x} = \Vsym{y}.
\end{gather*}

The correctness results hold for \Marpa input streams,
but to preserve the time complexity bounds,
restrictions must be imposed.
In stating them,
let it be understood that
\begin{equation*}
	\token{\tuple{\Vsym{x}, \Vloc{x}, \var{length} }} \in \var{tokens}
\end{equation*}
We require that,
for some constant \var{c1},
possibly dependent on the grammar \Cg{},
that
\begin{equation}
\label{e:restriction1}
\begin{gathered}
\univQ{ \VTtyped{sym}{SYM}, \VTtyped{x}{LOC}, \Vtyped{length}{\naturals}}{ \\
  \token{\tuple{\V{sym},\V{j},\V{length}}} \in \V{tokens} \implies \var{length} < \var{c1}}.
\end{gathered}
\end{equation}
We also require that
the cardinality of the set of tokens starting at any
one location
be less than a constant, call it \var{c2}:
\begin{equation}
\label{e:restriction2}
\begin{gathered}
    \univQ{ \VTtyped{j}{LOC} }{ \\
	\left|\,\left\lbrace\,
	    \begin{gathered}
	    \VTtyped{t}{TOK} \suchthat \\
	    \existQ{ \VTtyped{sym}{SYM}, \Vtyped{length}{\naturals}} { \\
		\V{t} = \tuple{\V{sym},\V{j},\V{length}} \land \V{t} \in \V{tokens}
		}
	    \end{gathered}
	  \,\right\rbrace\,\right|
	    < \V{c2}
    }
\end{gathered}
\end{equation}
Because the number of symbols is a constant depending on the grammar,
\eqref{e:restriction2} follows from \eqref{e:restriction1}.
Restrictions \ref{e:restriction1}
and \ref{e:restriction2}
impose little or no obstacle
to the practical use
of \Marpa's generalized input model.
And with them,
the complexity results for \Marpa{} stand.

\bibliographystyle{plain}

\hbadness=5000

\clearpage
\tableofcontents

\end{document}